\documentclass[10pt,a4paper,onecolumn]{article}

\usepackage[colorlinks=true,linkcolor=black,citecolor=black,urlcolor=black]{hyperref}       
\usepackage{amsfonts}       
\usepackage{fullpage}
\usepackage[mathscr]{eucal}
\usepackage{epsfig,epsf,psfrag}
\usepackage{amssymb,amsmath,amsfonts,latexsym}
\usepackage{graphicx}
\usepackage{epstopdf}
\usepackage{bm,xcolor,url}
\usepackage[caption=false]{subfig} 
\usepackage{fixltx2e}
\usepackage{array}
\usepackage{verbatim}
\usepackage{bm}
\usepackage{algpseudocode, cite}
\usepackage{algorithm}
\usepackage{verbatim}
\usepackage{textcomp}
\usepackage{mathrsfs}
\usepackage{relsize}
\usepackage{subfig}
\usepackage{amsthm}
\usepackage{enumerate}
\usepackage{setspace}
\usepackage{cases}
\usepackage{footnote}
\usepackage{tabularx,tabu}
\usepackage{tablefootnote}
\usepackage{threeparttable}
\usepackage{tikz}
\usetikzlibrary{datavisualization}
\usetikzlibrary{datavisualization.formats.functions}
\usetikzlibrary{decorations.pathreplacing}
\usepackage{textcomp}
\usepackage{multicol}
\usepackage{lipsum}
\usepackage{enumitem}
\usepackage{bbm}
\usepackage{stackengine}
\stackMath 
\catcode`~=11 \def\UrlSpecials{\do\~{\kern -.15em\lower .7ex\hbox{~}\kern .04em}} \catcode`~=13 

\allowdisplaybreaks[3]

\newcommand{\vecz}{\mathbf{0}}

\newcommand{\norm}[1]{\left\Vert#1\right\Vert}
\newcommand{\normt}[1]{\Vert#1\Vert}

\newcommand{\abs}[1]{\left\lvert#1\right\rvert}
\newcommand{\lrpar}[1]{\left(#1\right)}
\newcommand{\lrang}[1]{\left\langle#1\right\rangle}

\newcommand{\nn}{\nonumber}

\newcommand{\defeq}{\triangleq}

\newcommand\numberthis{\addtocounter{equation}{1}\tag{\theequation}} 
\newcommand{\prox}{\mathbf{prox}}

\newcommand{\tilell}{\widetilde{\ell}}

\newcommand{\barnabla}{\overline{\nabla}}

\newcommand{\barbh}{\overline{\bf h}}

\newcommand{\barbr}{\overline{\bf r}}

\newcommand{\barbz}{\overline{\bf z}}

\newcommand{\barbH}{\overline{\bf H}}

\newcommand{\barbR}{\overline{\bf R}}

\newcommand{\barbW}{\overline{\bf W}}

\newcommand{\calA}{\mathcal{A}}
\newcommand{\calB}{\mathcal{B}}
\newcommand{\calC}{\mathcal{C}}

\newcommand{\calE}{\mathcal{E}}
\newcommand{\calF}{\mathcal{F}}

\newcommand{\calH}{\mathcal{H}}

\newcommand{\calP}{\mathcal{P}}

\newcommand{\calR}{\mathcal{R}}
\newcommand{\calS}{\mathcal{S}}

\newcommand{\calU}{\mathcal{U}}

\newcommand{\calX}{\mathcal{X}}
\newcommand{\calY}{\mathcal{Y}}
\newcommand{\calZ}{\mathcal{Z}}

\newcommand{\ba}{\mathbf{a}}
\newcommand{\bA}{\mathbf{A}}
\newcommand{\bb}{\mathbf{b}}
\newcommand{\bB}{\mathbf{B}}

\newcommand{\bd}{\mathbf{d}}
\newcommand{\bD}{\mathbf{D}}

\newcommand{\bE}{\mathbf{E}}

\newcommand{\bG}{\mathbf{G}}
\newcommand{\bh}{\mathbf{h}}
\newcommand{\bH}{\mathbf{H}}

\newcommand{\bI}{\mathbf{I}}

\newcommand{\bL}{\mathbf{L}}

\newcommand{\br}{\mathbf{r}}

\newcommand{\bV}{\mathbf{V}}

\newcommand{\bW}{\mathbf{W}}
\newcommand{\bx}{\mathbf{x}}
\newcommand{\bX}{\mathbf{X}}
\newcommand{\by}{\mathbf{y}}
\newcommand{\bY}{\mathbf{Y}}
\newcommand{\bz}{\mathbf{z}}
\newcommand{\bZ}{\mathbf{Z}}

\newcommand{\bbE}{\mathbb{E}}

\newcommand{\bbR}{\mathbb{R}}

\newcommand{\barbbR}{\overline{\bbR}}

\DeclareMathAlphabet{\mathbsf}{OT1}{cmss}{bx}{n}
\DeclareMathAlphabet{\mathssf}{OT1}{cmss}{m}{sl}

\DeclareSymbolFont{bsfletters}{OT1}{cmss}{bx}{n}  
\DeclareSymbolFont{ssfletters}{OT1}{cmss}{m}{n}
\DeclareMathSymbol{\bsfGamma}{0}{bsfletters}{'000}
\DeclareMathSymbol{\ssfGamma}{0}{ssfletters}{'000}
\DeclareMathSymbol{\bsfDelta}{0}{bsfletters}{'001}
\DeclareMathSymbol{\ssfDelta}{0}{ssfletters}{'001}
\DeclareMathSymbol{\bsfTheta}{0}{bsfletters}{'002}
\DeclareMathSymbol{\ssfTheta}{0}{ssfletters}{'002}
\DeclareMathSymbol{\bsfLambda}{0}{bsfletters}{'003}
\DeclareMathSymbol{\ssfLambda}{0}{ssfletters}{'003}
\DeclareMathSymbol{\bsfXi}{0}{bsfletters}{'004}
\DeclareMathSymbol{\ssfXi}{0}{ssfletters}{'004}
\DeclareMathSymbol{\bsfPi}{0}{bsfletters}{'005}
\DeclareMathSymbol{\ssfPi}{0}{ssfletters}{'005}
\DeclareMathSymbol{\bsfSigma}{0}{bsfletters}{'006}
\DeclareMathSymbol{\ssfSigma}{0}{ssfletters}{'006}
\DeclareMathSymbol{\bsfUpsilon}{0}{bsfletters}{'007}
\DeclareMathSymbol{\ssfUpsilon}{0}{ssfletters}{'007}
\DeclareMathSymbol{\bsfPhi}{0}{bsfletters}{'010}
\DeclareMathSymbol{\ssfPhi}{0}{ssfletters}{'010}
\DeclareMathSymbol{\bsfPsi}{0}{bsfletters}{'011}
\DeclareMathSymbol{\ssfPsi}{0}{ssfletters}{'011}
\DeclareMathSymbol{\bsfOmega}{0}{bsfletters}{'012}
\DeclareMathSymbol{\ssfOmega}{0}{ssfletters}{'012}

\newcommand{\hatf}{\widehat{f}}

\newcommand{\tilg}{\widetilde{g}}

\newcommand{\hatbh}{\widehat{\bh}}

\newcommand{\hatbr}{\widehat{\br}}

\newcommand{\hatbV}{\widehat{\bV}}

\newcommand{\hatbW}{\widehat{\bW}}

\newcommand{\tilbW}{\widetilde{\bW}}

\newcommand{\barf}{\overline{f}}

\newcommand{\ubh}{\underline{\bh}}

\newcommand{\iid}{i.i.d.\ }

\newcommand{\floor}[1]{\lfloor{#1}\rfloor}
\newcommand{\lrangle}[2]{\left\langle{#1},{#2}\right\rangle}

\newcommand{\eqa}{\stackrel{\rm(a)}{=}}

\newcommand{\leb}{\stackrel{\rm(b)}{\le}}
\newcommand{\lec}{\stackrel{\rm(c)}{\le}}

\DeclareMathOperator*{\argmax}{arg\,max}
\DeclareMathOperator*{\argmin}{arg\,min}

\DeclareMathOperator{\tr}{tr}

\newtheorem{theorem}{Theorem} 
\newtheorem*{theorem*}{Theorem}
\newtheorem{lemma}{Lemma}

\newtheorem{corollary}{Corollary}

\theoremstyle{definition}

\theoremstyle{remark}
\newtheorem{remark}{Remark}

\newcommand{\qednew}{\nobreak \ifvmode \relax \else
      \ifdim\lastskip<1.5em \hskip-\lastskip
      \hskip1.5em plus0em minus0.5em \fi \nobreak
      \vrule height0.75em width0.5em depth0.25em\fi}

\usepackage{xr}
\usepackage{authblk}

\title{A Unified Framework for Stochastic Matrix Factorization via Variance Reduction}

\author[1,2,3]{Renbo Zhao}
\author[2]{William B. Haskell}
\author[1]{Jiashi Feng}
\affil[1]{\small Department of Electrical and Computer Engineering, National University of Singapore}
\affil[2]{Department of Industrial and Systems Engineering, National University of Singapore}
\affil[3]{Department of Mathematics, National University of Singapore}

\begin{document}

\maketitle

\begin{abstract}
We propose a unified framework to speed up the existing stochastic matrix factorization (SMF) algorithms via variance reduction. Our framework is general and it subsumes several  well-known SMF formulations in the literature. We perform a non-asymptotic convergence analysis of our framework and derive computational and sample complexities for our algorithm to converge to an $\epsilon$-stationary point in expectation. In addition, extensive experiments for a wide class of SMF formulations demonstrate that our framework consistently yields faster convergence and a more accurate output dictionary vis-\`a-vis  state-of-the-art frameworks. 
\end{abstract}

\section{Introduction}

Matrix factorization plays an important role in machine learning, due to its wide applications including collaborative filtering~\cite{Koren_09}, parts-based learning~\cite{Lee_99} and clustering~\cite{Ding_04}. 
Given a data matrix $\bY\!\in\!\bbR^{d\times n}$ (where $n$ denotes the number of data samples and $d$ denotes the dimension of each data sample), one aims to find matrices $\bW\!\in\!\bbR^{d\times k}$ and $\bH\!\in\!\bbR^{k\times n}$ such that $\bY\!\approx\! \bW\bH$, where usually $k\!\ll\! \min(d,n)$. Many existing algorithms in the literature  find $\bW$ and $\bH$ by minimizing $\normt{\bY-\bW\bH}_F^2$ (possibly with regularizations on $\bW$ or $\bH$ to induce structural properties). However, when the number of data samples $n$ becomes large, solving this problem using {batch algorithms} can be highly inefficient, in terms of both computation and storage. To improve the efficiency, {\em online (stochastic) matrix factorization} algorithms~\cite{Mairal_10,Spre_15,Guan_12,Feng_13,Shen_14b,Shen_16,Doh_16,Zhao_17,Zhao_17b} have been proposed to learn the {\em dictionary (matrix)} $\bW$ from a sequence of randomly drawn data samples (possibly in mini-batches).\footnote{We use ``online'' and ``stochastic'' interchangeably. In the literature, online algorithms also cover the setting where the number of data samples $n$ is infinite. However, we only consider $n$ to be finite, but can be very large.} Extensive numerical evidence has shown that the stochastic  matrix factorization algorithms can learn $\bW$ significantly faster than their batch counterparts, in terms of (empirical) convergence rate. In addition, only $O(dk)$ storage space  are consumed by these algorithms, in contrast to $O(ndk)$ by batch methods. Due to these advantages, stochastic matrix factorization algorithms have gained much popularity in recent years. 

From an optimization point of view, previous stochastic matrix factorization algorithms~\cite{Mairal_10,Spre_15,Guan_12,Feng_13,Shen_14b,Shen_16,Doh_16,Zhao_17,Zhao_17b} aim to solve a {\em nonconvex stochastic program} involving the dictionary $\bW$ (see~\eqref{eq:obj_SP} in Remark~\ref{rmk:stoc_prog}). They mainly leverage two optimization frameworks, namely {\em stochastic majorization-minimization} (SMM)~\cite{Mairal_13} and {\em stochastic (sub-)gradient descent} (SGD)~\cite{Borkar_08}. Based on either framework, almost sure asymptotic convergence to stationary points of the stochastic program has been shown. However, the asymptotic nature of this convergence analysis cannot provide insights into the dictionary learning process at any finite time instant. 
Thus, we wish to understand the non-asymptotic convergence of dictionary learning, as well as the 
sample complexity for learning a ``reasonable'' dictionary $\bW$. 
Besides, we desire improved stochastic methods that yield faster convergence, at least for large but {\em finite}  datasets. 

{\bf Main Contributions}.\ Inspired by the recent advances in variance reduction techniques (particularly for {\em nonconvex} problems)~\cite{Johnson_13,Reddi_16,AllenZhu_16a}, we propose a unified framework to speed up the existing stochastic matrix factorization  algorithms via variance reduction (VR). As shown in Section~\ref{sec:prob_form}, our framework is {\em general} and it subsumes {eight}  well-known SMF formulations as special instances, including those {robust} variants that explicitly model outliers in their objective functions. Within this framework, we derive a {\em non-asymptotic} convergence rate, measured in terms of how ``fast'' the algorithm converges to a stationary point in expectation (see Section~\ref{sec:analysis_main_results} for precise statements). 
Accordingly, we derive sample and computational complexities for our algorithm to converge to an $\epsilon$-stationary point in expectation. 
To the best of our knowledge, this is the {\em first time} that a {non-asymptotic convergence analysis} (together with {sample complexity}) is established for SMF algorithms. 
To further improve the efficiency of our framework, we also provide an {\em inexact} convergence analysis of our framework, where in each step the coefficient (and outlier) vectors are only learned {approximately}. In the context of {stochastic (variance-reduced)}  nonconvex proximal gradient algorithms, this is the {\em first} inexact analysis thus far. 
We show, via extensive experiments on various SMF formulations and datasets, that our variance-reduced framework consistently outperforms state-of-the-art frameworks (including SMM and SGD), in terms of both  convergence rate and accuracy of learned dictionaries. 


{\bf Related works}.\ 
Another line of works~\cite{Krish_13,Jin_16,Wang_17} considers using SMF algorithms in a matrix completion setting, where the matrices $\bW$ and $\bH$ are jointly learned from a random sequence of entries (or columns) of a data matrix $\bY$. By leveraging low-rankness and the restricted isometry property (RIP) on $\bY$, the product of $\bW$ and $\bH$ can {\em exactly} recover $\bY$ (with high probability). However, their problem setting differs from ours in two aspects. First, we do not store or explicitly learn $\bH$, due to space constraints. Second, our convergence results do not rely on low-rankness or RIP. Therefore, similar to~\cite{Mairal_10,Spre_15,Guan_12,Feng_13,Shen_14b,Shen_16,Doh_16,Zhao_17,Zhao_17b}, our results are stated in terms of stationary points, which are oftentimes appealing in real applications~\cite{Mairal_10,Zhao_17}. In the context of stochastic PCA, VR techniques have been applied to yield a linear convergence~\cite{Shamir_15b}. Again, the problem setting therein vastly differs from ours. 

{\bf Notations}.\ We use boldface capital letters, boldface lowercase letters and lowercase letters to denote matrices, column vectors and scalars respectively. 
Denote $\bbR_+$ as $[0,\infty)$ and $\barbbR_+$ as $[0,\infty]$. For any $n\ge 1$, define $[n]\defeq \{1,\ldots,n\}$ and $(n]\defeq \{0,\ldots,n\}$. 
For a matrix $\bX$, denote its $j$-th column as $\bX_{:j}$. We use $\norm{\cdot}$ to denote the $\ell_2$-norm of a vector and the Frobenius norm of  a matrix. For a convex and closed set $\calX$, denote its indicator function as $\delta_\calX$, such that $\delta_\calX(\bx)=0$ if $\bx\in\calX$ and $\infty$ otherwise. 
All the sections, lemmas and algorithms beginning with `S' will appear in the supplemental material. 

\section{Problem Formulation}\label{sec:prob_form}

Let $\{\by_i\}_{i=1}^n\subseteq\bbR^{d}$ denote the set of data samples. 
We first define a loss function $\ell_1$ w.r.t.\ a data sample $\by$ and a dictionary $\bW\in\bbR^{d\times k}$ as 
\begin{equation}
\ell_1(\by,\bW) \defeq \min_{\bh\in\calH}\tilell_1(\by,\bW,\bh),\quad 
 \tilell_1(\by,\bW,\bh)\defeq\frac{1}{2} \norm{\by-\bW\bh}_2^2 + \varphi(\bh), \label{eq:loss_1}
\end{equation}
where $\calH$ and $\varphi:\bbR^{k}\to\barbbR_+$ denote the constraint set and regularizer of the coefficient vector $\bh$ respectively. 
Since large-sale datasets  may contain outliers, we can define a ``robust'' version of $\ell_1$ as 
\begin{equation}
\ell_2(\by,\bW) \defeq \min_{\bh\in\calH,\br\in\calR}\tilell_2(\by,\bW,\bh,\br), \quad
\tilell_2(\by,\bW,\bh,\br)\defeq\frac{1}{2} \norm{\by-\bW\bh-\br}_2^2 + \varphi(\bh) + \phi(\br),\label{eq:loss_2}
\end{equation}
where $\calR$ and $\phi:\bbR^{d}\to\barbbR_+$ denote the constraint set and regularizer of the outlier vector $\br$ respectively. For convenience, let $\ell$ denote either $\ell_1$ or $\ell_2$. 
Based on $\ell$, we formulate our problem as 
\begin{equation}
\min_{\bW\in\calC} \left[f(\bW)\defeq g(\bW)+\psi(\bW)\right], \quad g(\bW)\defeq\frac{1}{n} \sum_{i=1}^n \ell(\by_i,\bW),\label{eq:obj}   
\end{equation}
where $\calC$ and $\psi:\bbR^{d\times k}\to\barbbR_+$ denote the constraint set and regularizer  of the dictionary $\bW$ respectively. Our targeted problem \eqref{eq:obj} is general and flexible, in the sense that by choosing the constraint sets $\calC$, $\calH$ and $\calR$ and the regularizers $\psi$, $\varphi$ and $\phi$ in different manners, \eqref{eq:obj} encompasses many important examples in the literature of SMF. For loss function $\ell_1$, we have: 
\begin{enumerate}[label={\bf (P\arabic*)},ref=\textbf{(P\arabic*)},topsep=-.1cm,itemsep=.1cm]
\item {\em Online DL (ODL)}~\cite{Mairal_10}: $\calC\defeq\{\bW\in\bbR^{d\times k}:\norm{\bW_{:i}}\le 1, \forall i\in[n]\}$, $\calH\defeq \bbR^k$, $\psi\equiv 0$ and $\varphi\defeq \lambda\norm{\cdot}_1$ ($\lambda>0$). Some other variants and extensions were also discussed in \cite[Section~5]{Mairal_10}. \label{prob:ODL}
\item {\em Online Structured Sparse Learning (OSSL)}~\cite{Spre_15}: $\calC\defeq \bbR^{d\times k}$, $\calH\defeq \bbR^k$, $\psi\equiv 0$ and $\varphi(\bh)\defeq \sum_{j\in[n']}\lambda_j\normt{\ubh_j}$, where $\ubh_j$ is a subvector of $\bh$, $\lambda_j\ge 0$ and $n'\in [n]$.\label{prob:OSSL} 
\item {\em Online NMF (ONMF)}~\cite{Guan_12}: $\calC\defeq \{\bW\in\bbR_+^{d\times k}\!:\! \norm{\bW_{:i}}_1=1,\forall i\!\in\![n]\}$, $\calH = \bbR_+^k$, $\psi\equiv 0$ and $\varphi\defeq \frac{\lambda}{2}\norm{\cdot}^2$ ($\lambda>0$). See \cite[Section~III]{Guan_12} for some extensions. \label{prob:ONMF}
\item {\em Smooth Sparse ODL (SSODL)}~\cite{Doh_16}: $\calC\defeq \{\bW\in\bbR_+^{d\times k}\!:\! \norm{\bW_{:i}}_1=1,\forall i\!\in\![n]\}$, $\calH = \bbR_+^k$, $\psi(\bW)\defeq \frac{\lambda_1}{2}\tr(\bW^T\bL\bW)$ and $\varphi\defeq \frac{\lambda_2}{2}\norm{\cdot}^2$, where $\lambda_1,\lambda_2>0$ and $\bL$ is positive semidefinite. \label{prob:SSODL} 
\end{enumerate}
\vspace{.1cm} For loss function $\ell_2$, we have: 
\begin{enumerate}[label={\bf (P\arabic*)},ref=\textbf{(P\arabic*)},resume,topsep=-.1cm,itemsep=.1cm]
\item  {\em Online Robust PCA (ORPCA)}~\cite{Feng_13}: $\calC\defeq \bbR^{d\times k}$, $\calH\defeq \bbR^k$, $\calR \defeq \bbR^d$, $\psi\defeq\frac{\lambda_1}{2n}\norm{\cdot}^2$, $\varphi\defeq \frac{\lambda_1}{2}\norm{\cdot}^2$ and $\phi\defeq \lambda_2\norm{\cdot}_1$, where $\lambda_1,\lambda_2>0$.\label{prob:ORPCA}
\item  {\em Online Max-norm Regularized Matrix Decomposition (OMRMD)}~\cite{Shen_14b}: $\calC\defeq \bbR^{d\times k}$, $\calH\defeq \{\bh\in\bbR^k:\norm{\bh}\le 1\}$, $\calR \defeq \bbR^d$, $\psi\defeq \frac{\lambda_1}{2n} \norm{\cdot}_{2,\infty}^2$, $\varphi\equiv 0$ and $\phi=\lambda_2\norm{\cdot}_1$, where $\lambda_1,\lambda_2>0$. \label{prob:OMRMD}
\item {\em Online Low-Rank Representation (OLRR)}~\cite{Shen_16}: $\calC\defeq \bbR^{d\times k}$, $\calH\defeq \bbR^k$, $\calR \defeq \bbR^d$, 
$\varphi\defeq \frac{\lambda_1}{2}\norm{\cdot}^2$, $\phi\defeq \lambda_2\norm{\cdot}_1$ and $\psi(\bW)\defeq\frac{1}{2n^2}(\normt{\bW^T\bA\bY}^2+\lambda_3\normt{\bB\bW}^2)$, where $\bY$, $\bA$ and $\bB$ are given (constant) matrices, and $\lambda_1,\lambda_2,\lambda_3>0$.\label{prob:OLRR}
\item  {\em Online Robust NMF (ORNMF)}~\cite{Zhao_17}: $\calC\defeq \{\bW\in\bbR_+^{d\times k}\!:\! \norm{\bW_{:i}}\le 1,\forall i\!\in\![n]\}$, $\calH\defeq [0,M]^k$, $\calR\defeq [-M',M']^d$, $\psi\equiv 0$, $\varphi\equiv 0$ and $\phi\defeq \lambda\norm{\cdot}_1$, where $M,M',\lambda>0$. \label{prob:ORNMF}
\end{enumerate}

\begin{remark}\label{rmk:stoc_prog}\vspace{.1cm}
In almost all the works cited in this section, the optimization problem 
is formulated as a {\em stochastic program} that minimizes the expected risk of dictionary training. It has the form 
\begin{equation}
\min_{\bW\in\calC} \left[\barf(\bW)\defeq \bbE_{\by\sim\nu}[\ell(\by,\bW)]+\psi(\bW)\right], \label{eq:obj_SP} 
\end{equation}
where $\nu$ is a probability distribution. The reasons that we consider~\eqref{eq:obj} rather than~\eqref{eq:obj_SP} are three-fold. First, \eqref{eq:obj} can be regarded as the {\em sample average approximation}~\cite{Kley_02} of~\eqref{eq:obj_SP}. Solving this approximation has been a popular and efficient approach for solving a general stochastic program like~\eqref{eq:obj_SP}~\cite{Shap_05}. Second, \eqref{eq:obj} can also be interpreted as minimizing the empirical risk incurred by any large-scale finite training dataset. Solving~\eqref{eq:obj} efficiently can greatly speed up the dictionary learning process. Last but not least, since \eqref{eq:obj} is a finite-sum minimization problem, we can employ the recently developed {\em variance reduction} techniques~\cite{Johnson_13,Reddi_16} to solve it in an efficient manner. See Section~\ref{sec:algo} for details. 
\end{remark}\vspace{-.2cm}

For algorithm and convergence analysis, we will focus on the loss function $\ell_2$, which contains $\ell_1$ as a special case (by choosing $\calR=\{\vecz\}$). 

\vspace{-.15cm}
\section{Algorithm}\label{sec:algo}
\vspace{-.1cm}

We present the pseudo-code for our stochastic matrix factorization algorithm with variance reduction techniques in Algorithm~\ref{algo:SMF_VR}.  To develop our algorithm, we first make two mild assumptions. We then describe our algorithm in details, with focus on learning the coefficient and outlier vectors. 

{\bf Assumptions}.\  The following two assumptions are satisfied by all the Problems \ref{prob:ODL} to \ref{prob:ORNMF}. 
\begin{enumerate}[label={\bf (A\arabic*)},ref=\textbf{(A\arabic*)},topsep=-.1cm,itemsep=.0cm]
\item $\psi$, $\varphi$ and $\phi$ are convex, proper and closed functions on $\bbR^{d\times k}$, $\bbR^k$ and $\bbR^d$ respectively, with proximal operators that can be evaluated efficiently.\footnote{For a function $\gamma:\bbR^d\to\bbR$, the ``efficient'' evaluation of $\prox_\gamma$ means that for any $\bx\in\bbR^d$, $\prox_\gamma(\bx)$ can be computed in $O(d)$ arithemetic operations. The same definitions apply to ``efficient projections'' in \ref{assump:convex_set}. }  \label{assump:regularizers}
\item $\calC$, $\calH$ and $\calR$ are convex and admit efficient (Euclidean) projections onto them. \label{assump:convex_set}
\end{enumerate}

\vspace{.1cm}{\bf Algorithm Description}.\ Variance-reduced stochastic optimization algorithms typically contain outer and inner loops~\cite{Johnson_13,Defazio_14}. In Algorithm~\ref{algo:SMF_VR}, we use $s$ and $t$ to denote the indices for outer and inner iterations respectively. At the beginning of each outer iteration $s$, we solve $n$ (regularized) least-square regression problems to learn the set of coefficient and outlier vectors $\{(\bh^{s,0}_j,\br^{s,0}_j)\}_{j=1}^n$ based on $\bW^{s,0}$. Based on the learned vectors, we compute $\bG_{s,0}$, which will be shown to be $\nabla g(\bW^{s,0})$ 
in Section~\ref{sec:conv_analysis}. 
In each inner iteration $t$, we solve a (random) mini-batch of regression problems, based on $\bW^{s,0}$ and $\bW^{s,t}$ respectively. The regression results, together with $\bG_{s,0}$, are used to calculate a search direction $\bV_{s,t}$, which acts as a variance-reduced gradient estimate of the loss function $g$ at $\bW^{s,t}$. 
The proximal operator of $\eta\psi+\delta_\calC$ either has a closed form~\cite{Yu_13} or can be obtained using Douglas-Rachford splitting or alternating direction method of multipliers~\cite{Parikh_14}, both in $O(dk)$ (arithmetic) operations. 
We choose the final dictionary via option I or II (in lines~\ref{line:opt_I} and \ref{line:opt_II} respectively).\footnote{In this work, the line numbers always refer to those in Algorithm~\ref{algo:SMF_VR}.} In practice, we choose the final dictionary as the last iterate $\bW^{s-1,m}$, i.e., option II. However, 
to streamline the analysis, we output the dictionary using option I. This is consistent with the many works on variance-reduced SGD algorithms, e.g., \cite{Johnson_13}. 

{\bf Learning coefficient and outlier vectors}.\ It can be seen that the computational complexity of Algorithm~\ref{algo:SMF_VR} crucially depends on solving the problem~\eqref{eq:loss_2}, which appears 
in lines~\ref{line:solve_hr1}, \ref{line:solve_hr2} and \ref{line:solve_hr3} of Algorithm~\ref{algo:SMF_VR}. To solve this problem, we leverage the {\em block successive upperbound minimization} framework~\cite{Raza_13}, by alternating between the following two steps, i.e., 
\begin{align}
\bh^+ &= \prox_{\varphi/L'+\delta_\calH}(\bh-(1/L')\bW^T(\bW\bh+\br-\by)), \;L'\defeq \norm{\bW}_2^2 \quad\mbox{and}\\
\br^+ &= \prox_{\phi+\delta_\calR}(\by-\bW\bh^+),
\end{align}
where $\bh^+$ and $\br^+$ denote the updated values of $\bh$ and $\br$ respectively (in one iteration). The proximal operators of $\varphi/L'+\delta_\calH$ and $\phi+\delta_\calR$ can be obtained in the same way as that of $\eta\psi+\delta_\calC$. 
Overall, the computational complexity of solving~\eqref{eq:loss_2} is $O(dk)$. 

\begin{remark}
In our implementation, we do not store $\{\bh^{s,0}_j,\br^{s,0}_j\}_{j=1}^n$ before computing $\bG_{s,0}$. Instead,\vspace{-.1cm} after obtaining $\bh^{s,0}_j$ and $\br^{s,0}_j$, we add them to the running sum in line~\ref{line:comput_G}, and then discard them. The same procedure applies to computing $\bV_{s,t}$. Thus the storage complexity of Algorithm~\ref{algo:SMF_VR} is $O(dk)$. This complexity is {\em linear} in the dimension of the optimization variable $\bW$, which has $dk$ entries. 
\end{remark}

\begin{algorithm}[t]
\caption{Stochastic Matrix Factorization via Variance Reduction} \label{algo:SMF_VR}
\begin{algorithmic}[1]
\State {\bf Input}: Number of inner iterations $m$, number of outer iterations $S$, mini-batch size $b$, step size $\eta$ 
\State {\bf Initialize} dictionary $\bW^{0,0}\in\calC$ 
\State {\bf For} $s=0,1,\ldots,S-1$
\State \quad Solve $(\bh^{s,0}_j,\br^{s,0}_j)\defeq \argmin_{\bh\in\calH,\br\in\calR} \tilell_2(\by_j,\bW^{s,0},\bh,\br)$, for any $j\!\in\![n]$. \label{line:solve_hr1}
\State \quad $\bG_{s,0} := \frac{1}{n}\sum_{j=1}^n(\bW^{s,0}\bh_j^{s,0}+\br_j^{s,0}-\by_j)(\bh_j^{s,0})^T$ \label{line:comput_G}
\State \quad{\bf For} $t = 0, 1, \ldots, m-1$ 
\State \quad\quad Uniformly randomly sample $\calB_{s,t}\subseteq[n]$ of size $b$ without replacement 
\State \quad\quad Solve $(\barbh^{s,t}_j,\barbr^{s,t}_j)\defeq \argmin_{\bh\in\calH,\br\in\calR} \tilell_2(\by_{j},\bW^{s,0},\bh,\br)$, for any $j\!\in\!\calB_{s,t}$.\label{line:solve_hr2}
\State \quad\quad Solve $(\bh^{s,t}_j,\br^{s,t}_j)\defeq \argmin_{\bh\in\calH,\br\in\calR} \tilell_2(\by_{j},\bW^{s,t},\bh,\br)$, for any $j\!\in\!\calB_{s,t}$. \label{line:solve_hr3}
\State \quad\quad $\bV_{s,t}:=\frac{1}{b}\sum_{j\in\calB_{s,t}}(\bW^{s,t}\bh_j^{s,t}+\br_j^{s,t}-\by_j)(\bh_j^{s,t})^T-(\bW^{s,0}\barbh_j^{s,t}+\barbr_j^{s,t}-\by_j)(\barbh_j^{s,t})^T+\bG_{s,0}$\label{line:V_st}
\State \quad\quad $\bW^{s,t+1}:=\prox_{\eta\psi+\delta_\calC}(\bW^{s,t}-\eta\bV_{s,t})$ \label{line:prox_W}
\State \quad {\bf End for}
\State \quad $\bW^{s+1,0}:= \bW^{s,m}$
\State {\bf End for} 
\State {\bf Option I}: Choose the final dictionary $\bW_{\rm final}$ uniformly randomly from $\{\bW^{s,t}\}_{s\in (S-1],t\in[m]}$\label{line:opt_I}
\State {\bf Option II}: Choose the final dictionary as $\bW^{s-1,m}$ \label{line:opt_II} 
\end{algorithmic}
\end{algorithm}\vspace{-.2cm}

\vspace{-.1cm}
\section{Convergence Analysis}\label{sec:conv_analysis}
\vspace{-.1cm}

{\bf Additional assumptions}.\ For analysis purposes, we make the following three additional assumptions:
\begin{enumerate}[label={\bf (A\arabic*)},ref=\textbf{(A\arabic*)},resume,topsep=-.0cm,itemsep=.0cm]
\item The data samples $\{\by_i\}_{i=1}^n$ lie in a compact set $\calY\subseteq\bbR^d$. \label{assump:data_bounded}
\item For any $\by\in\calY$ and $\bW\in\calC$, 
$\tilell_2(\by,\bW,\cdot,\cdot)$ is {\em jointly} strongly convex on $\calH\times \calR$. \label{assump:sc}
\item $\calC$ is compact; $\calH$ is compact {\em or} $\varphi$ is coercive on $\bbR^k$; $\calR$ is compact {\em or} $\varphi$ is coercive on $\bbR^d$. \label{assump:compact_set}
\end{enumerate}
\begin{remark}
Assumption \ref{assump:data_bounded} is natural since all real data are bounded. 
Assumption \ref{assump:sc} 
is common in the literature~\cite{Mairal_10,Zhao_17} and can be satisfied by adding Tikhonov regularizers (see~\cite[Remark~4]{Zhao_17} for a detailed discussion). 
Assumption \ref{assump:compact_set} is satisfied by all the Problems \ref{prob:ODL} to \ref{prob:ORNMF}, except in \ref{prob:OSSL}, \ref{prob:ORPCA}, \ref{prob:OMRMD} and \ref{prob:OLRR}, $\calC$ is not bounded. However, for each problem, it was proven that the sequence of basis matrices $\{\bW_i\}_{i\ge 0}$ generated by the corresponding algorithm in \cite{Feng_13}, \cite{Shen_14b} or \cite{Shen_16} belonged to a compact set almost surely. Thus the compactness of $\calC$ is a reasonable assumption. 
\end{remark}\vspace{-.15cm}

\subsection{Key Lemmas and Main Results}\label{sec:analysis_main_results}
\vspace{-.1cm}
We first present some definitions that are used in the statements of our results. We then state our key lemmas (Lemmas~\ref{lem:regularity} and \ref{lem:var_bound}) and main convergence theorems (Theorems~\ref{thm:main} and \ref{thm:inexact}), together with the derived sample and computational complexities (Corollary~\ref{cor:complexity}). The proof sketches of all the results are shown in Section~\ref{sec:proof_sketch}, with complete proofs deferred to Section~\ref{sec:proof_main_thm} to Section~\ref{sec:proof_var_bound}. 

{\bf Definitions}.\ 
Define a filtration $\{\calF_{s,t}\}_{s\in (S-1],t\in[m]}$ such that $\calF_{s,t}$ is the $\sigma$-algebra generated by the random sets $\{\calB_{s',t'}\}_{s'\in (s-1],t'\in[m]}\cup\{\calB_{s,t'}\}_{t'\in [t-1]}$. 
It is easy to see that $\{\bW^{s,t}\}_{s\in (S-1],t\in[m]}$ is adapted to to $\{\calF_{s,t}\}_{s\in (S-1],t\in[m]}$. 
Consider a composite function $h\!\defeq\! h_1\!+\!h_2$ defined on a Euclidean space $\calE$, where $h_1\!:\!\calE\!\to\!\bbR$ is differentiable and $h_2\!:\!\calE\!\to\!\barbbR_+$ is proper, closed and convex. 
For any $\eta\!>\!0$ and $\bx\!\in\!\calE$, define the (proximal) {\em gradient mapping}~\cite{Nest_04} of $h$ at $\bx$ with step size $\eta$, $\Gamma_{h,\eta}(\bx)\!\defeq\! (1/\eta)(\bx\!-\!\prox_{\eta h_2}(\bx\!-\!\eta\nabla h_1(\bx)))$. A point $\by\in\calE$ is {\em $\epsilon$-stationary} w.r.t.\ the problem $\min_{\bx\in\calE}\!h(\bx)$ and $\eta$ if 
$\normt{\Gamma_{h,\eta}(\by)}^2\le \epsilon$. Finally, define $\alpha(n,b)\defeq (n-b)/(b(n-1))$. 

\begin{lemma}[Smoothness of $\ell_2$]\label{lem:regularity}
Define $(\bh^*(\by,\bW),\br^*(\by,\bW))\defeq \argmin_{\bh\in\calH,\br\in\calR}\tilell_2(\by,\bW,\bh,\br)$. 
The loss function $\ell_2(\cdot,\cdot)$ in \eqref{eq:loss_2} is differentiable on $\bbR^d\times\bbR^{d\times k}$ and 
\begin{equation}
\nabla_\bW\ell_2(\by,\bW) = \left(\bW\bh^*(\by,\bW)+\br^*(\by,\bW)-\by\right)\bh^*(\by,\bW)^T. \label{eq:deriv_ell2}
\end{equation}
In addition, $(\by,\bW)\mapsto\nabla_\bW\ell_2(\by,\bW)$ is continuous on $\calY\times \calC$. Moreover, for any $\by\in\calY$, $\bW\mapsto\nabla_\bW\ell_2(\by,\bW)$ is Lipschitz on $\calC$ with a Lipschitz constant $L$ independent of $\by$. 
\end{lemma}\vspace{-.0cm}

\begin{lemma}[Variance bound of $\bV_{s,t}$]\label{lem:var_bound}
In Algorithm~\ref{algo:SMF_VR}, $\bbE_{\calB_{s,t}}[\bV_{s,t}\vert\calF_{s,t}]=\nabla g(\bW^{s,t})$ and 
\begin{align}
\bbE_{\calB_{s,t}}[\normt{\bV_{s,t}-\nabla g(\bW^{s,t})}^2\vert\calF_{s,t}]\le \alpha(n,b)L^2\normt{\bW^{s,t}-\bW^{s,0}}^2, \;\forall s\ge 0,\;\forall t\in(m]. \label{eq:var_bound}
\end{align}
\end{lemma}\vspace{-.1cm}

{\bf Measuring convergence rate}.\ 
Since the loss function $\ell$ in~\eqref{eq:obj} is nonconvex (due to bilinearity), obtaining global minima of $f$ is in general out-of-reach. Thus in almost all the previous works~\cite{Mairal_10,Spre_15,Guan_12,Feng_13,Shen_14b,Shen_16,Doh_16,Zhao_17,Zhao_17b}, convergence to {\em stationary points} of~\eqref{eq:obj} was studied instead. Define $f'\defeq f+\delta_\calC$. Following~\cite{Ghad_16a,Reddi_16}, for the sequence of (random) dictionaries $\{\bW^{s,t}\}_{s\in (S-1], t\in[m]}$ generated in Algorithm~\ref{algo:SMF_VR}, we propose to use $\bbE[\normt{\Gamma_{f',\eta}(\bW^{s,t})}^2]$ to measure its convergence rate (to stationary points).
To see the validity of this measure, define the directional derivative of $f$ at $\bW\in\calC$ along $\bD$, i.e., $\barnabla f(\bW;\bD)\defeq \lim_{\delta\downarrow 0} (f(\bW+\delta \bD)-f(\bW))/{\delta}$.  
In previous works~\cite{Mairal_10,Spre_15,Guan_12,Feng_13,Shen_14b,Shen_16,Doh_16,Zhao_17,Zhao_17b}, $\bW'\in\calC$ is characterized as a stationary point of~\eqref{eq:obj} if and only if for any $\bW\in\calC$, $\barnabla f(\bW';\bW-\bW')\ge 0$. 
It can be shown that this condition is equivalent to $\Gamma_{f',\eta}(\bW')=\vecz$, for any $\eta>0$. In particular, if $\psi\equiv0$, we can verify that both conditions are equivalent to the condition that $-\nabla g(\bW')$ lies in 
the {\em normal cone} of $\calC$ at $\bW'$.   

Equipped with the proper convergence measure, we now state our main convergence theorem. 

\begin{theorem}\label{thm:main}
Assume that \ref{assump:regularizers} to \ref{assump:compact_set} hold. In Algorithm~\ref{algo:SMF_VR}, if we choose $\eta=1/(\theta L)$ for some $\theta>2$ and $m$ and $b$ to satisfy $\theta(\theta-1)\ge 2m(m+1)\alpha(n,b)$, then
\begin{equation}
\bbE[\normt{\Gamma_{f',\eta}(\bW_{\rm final})}^2] \le \left(\frac{f(\bW^{0,0})-f^*}{\eta\left({1}/{2}-\eta L\right)}\right)\frac{1}{mS},\quad\mbox{where}\quad f^*\defeq\min_{\bW\in\calC} f(\bW).  \label{eq:main_thm}
\end{equation}
\end{theorem}\vspace{-.2cm}
Note that $mS$ denotes the total number of inner iterations. Thus \eqref{eq:main_thm} states a sublinear convergence rate of $\{\bW^{s,t}\}_{s\ge 0, t\in[m]}$ towards the set of stationary points of~\eqref{eq:obj}. In addition, since $\alpha(n,b)\le 1/b$, we can choose $\theta$, $m$ and $b$ such that $m\!\le\!\! \sqrt{b\theta(\theta\!-\!1)/2}\!-\!1\,(\theta\!>\!2)$ to satisfy the conditions in Theorem~\ref{thm:main}.\vspace{.0cm} 

In practice, for efficiency considerations, it is important that we learn the coefficient and outlier vectors (in lines~\ref{line:solve_hr1}, \ref{line:solve_hr2} and \ref{line:solve_hr3}) {\em approximately}, by only finding {\em inexact} solutions to the subproblem~\eqref{eq:loss_2}. The errors in these inexact solutions will be propagated to the variance-reduced gradient $\bV_{s,t}$ as a whole. In the next theorem, we show that if the (additive) error $\bE_{s,t}$ in $\bV_{s,t}$ is properly controlled (i.e., $\normt{\bE_{s,t}}$ decreases at a certain rate), then  we can achieve the same convergence rate as in Theorem~\ref{thm:main}. Note that although inexact analyses have been done for {\em batch}~\cite{Schmidt_11} and {\em incremental}~\cite{Sra_12} proximal gradient (PG) algorithms, our result is the first one for a {\em stochastic} (variance-reduced) PG algorithm. 

\begin{theorem}[Inexact Analysis]\label{thm:inexact}
Assume that \ref{assump:regularizers} to \ref{assump:compact_set} hold. In Algorithm~\ref{algo:SMF_VR}, if $\eta=1/(\theta L)$ $(\theta\!>\!2)$, $\theta(\theta-1)\!\ge\! 4m(m+1)\alpha(n,b)$ and $\normt{\bE_{s,t}}\!=(ms+t)^{-(1/2+\tau)}$ for any $\tau>0$, then $\bbE[\normt{\Gamma_{f',\eta}(\bW_{\rm final})}^2] \!=\!O(1/(mS))$. 
\end{theorem}

\begin{remark}[Bound $\normt{\bE_{s,t}}$]
Denote an approximate solution of $\argmin_{\bh\in\calH,\br\in\calR} \tilell_2(\by_j,\bW^{s,t},\bh,\br)$\vspace{-.08cm} as $(\hatbh^{s,t}_j,\hatbr^{s,t}_j)$. From the definition of $\bV_{s,t}$ and the compactness of $\calH$ and $\calR$, we see that $\normt{\bV_{s,t}}$ can\vspace{-.08cm} be bounded by the learning errors in coefficient and outlier vectors, i.e.,\vspace{-.05cm} $\{(\normt{\bh^{s,t}_j-\hatbh^{s,t}_j},\normt{\br^{s,t}_j-\hatbr^{s,t}_j})\}_{j\in\calB_{s,t}}$.\vspace{-.05cm} For each $j\!\in\![n]$, these learning errors can 
 be further bounded by the (infimum of) norm of (sub-)gradient of $\tilell_2(\by_j,\bW^{s,t},\cdot,\cdot)$ at $(\hatbh^{s,t}_j,\hatbr^{s,t}_j)$, due to the first-order optimality conditions.\vspace{-.05cm} Therefore, we are able to bound $\normt{\bE_{s,t}}$ in terms of the approximate solutions $\{(\hatbh^{s,t}_j,\hatbr^{s,t}_j)\}_{j=1}^n$. 
\end{remark}

Finally, based on the convergence rates in Theorems~\ref{thm:main} and \ref{thm:inexact}, we are able to derive the sample and computational complexities for Algorithm~\ref{algo:SMF_VR} to attain an $\epsilon$-stationary point (in expectation).

\begin{corollary}[Sample and Computational Complexities]\label{cor:complexity}
For any $\epsilon\!>\!0$ and $\theta\!>\!2$, 
the sample and computational complexities for $\bW_{\rm final}$ in Algorithm~\ref{algo:SMF_VR} to be an $\epsilon$-stationary point of~\eqref{eq:obj} in expectation (i.e., $\bbE[\normt{\Gamma_{f',\eta}(\bW_{\rm final})}^2]\le \epsilon$) are $O(n^{2/3}/\epsilon)$ and $O(n^{2/3}dk/\epsilon)$ respectively.
\end{corollary}\vspace{-.2cm}

\subsection{Proof Sketches}\label{sec:proof_sketch}


{\bf Proof Sketch of Lemma~\ref{lem:regularity}}.\
 By Assumption~\ref{assump:compact_set}, if $\calH$ is not compact, then $\varphi$ is coercive. This implies the boundedness of $\bh^*(\by,\bW)$. A similar argument also applies to $\br^*(\by,\!\bW)$.\vspace{-.05cm} Thus it is equivalent to minimizing $\tilell_2$ over a compact set $\calH'\!\times\!\calR'\subseteq\calH\!\times\!\calR$. 
In addition, Assumptions~\ref{assump:data_bounded} and \ref{assump:compact_set} ensure the compactness of $\calY$ and $\calC$ respectively.  
Since Assumptions~\ref{assump:convex_set} and \ref{assump:sc} ensure the uniqueness of $(\bh^*(\by,\bW),\br^*(\by,\bW))$, for any $(\by,\bW)\in\bbR^d\times\bbR^{d\times k}$, we can invoke Danskin's theorem (see Lemma~\ref{lem:Danskin}) to guarantee the differentiability of $\ell_2$, and compute 
$\nabla_\bW\ell_2(\by,\bW)$ as in~\eqref{eq:deriv_ell2}. 
Additionally, we can invoke the Maximum theorem (see Lemma~\ref{lem:maximum}) to ensure the continuity of  $(\by,\bW)\!\mapsto\!(\bh^*(\by,\bW),\br^*(\by,\bW))$ on $\calY\!\times\!\calC$. This implies the continuity of $(\by,\bW)\!\mapsto\!\nabla_\bW\ell_2(\by,\bW)$. Based on this, we can assert the Lipschitz continuity of both $\bh^*(\cdot,\cdot)$ and $\br^*(\cdot,\cdot)$ on $\calY\!\times\!\calC$. The proceeding arguments, together with the compactness of $\calY$, $\calC$, $\calH'$ and $\calR'$, allow us to conclude the Lipschitz continuity of $\bW\!\mapsto\!\nabla_\bW\ell_2(\by,\bW)$. \qed

{\bf Proof Sketch of Lemma~\ref{lem:var_bound}}.\
To show~\eqref{eq:var_bound}, we first apply Lemma~\ref{lem:uni_samp_wo} to $\normt{\bV_{s,t}-\nabla g(\bW^{s,t})}^2$ and then make use  of the $L$-Lipschitz continuity of $\bW\!\mapsto\!\nabla_\bW\ell_2(\by,\bW)$ in Lemma~\ref{lem:regularity}.  
\qed




{\bf Proof Sketch of Theorem~\ref{thm:main}}.\ From Lemma~\ref{lem:regularity}, we know the loss function $g$ is differentiable. 
Define $\psi'\!\defeq\! \psi+\delta_\calC$ and $\tilbW^{s,t+1}\!\defeq\!\prox_{\eta\psi'}(\bW^{s,t}\!-\!\eta\nabla g(\bW^{s,t}))$, then $\Gamma_{f',\eta}(\bW^{s,t})\!=\!({1}/{\eta})(\bW^{s,t}\!-\!\tilbW^{s,t+1})$. Using Lemma~\ref{lem:Lips_bound} and the Lipschitz continuity of $\nabla g$ (see Lemma~\ref{lem:regularity}), we have the recursion
\begin{align}
f'(\bW^{s,t+1}) &\le f'(\bW^{s,t}) + \langle\bW^{s,t+1}-\tilbW^{s,t+1},\nabla g(\bW^{s,t})-\bV_{s,t}\rangle\!-\!{1}/(2\eta) \normt{\bW^{s,t+1}\!-\!\tilbW^{s,t+1}}^2\nn\\
 &\hspace{.0cm}+ \left({L}/{2}-{1}/(2\eta)\right)\normt{\bW^{s,t+1}-\bW^{s,t}}^2 +\left(L-{1}/(2\eta)\right)\normt{\tilbW^{s,t+1}\!-\!\bW^{s,t}}^2 . \label{eq:succ_inner_main}
\end{align}
Conditioning on $\calF_{s,t}$, we then take expectations w.r.t.\ $\calB_{s,t}$ on both sides of~\eqref{eq:succ_inner_main}. Using Lemma~\ref{lem:var_bound}, we can bound $\bbE_{\calB_{s,t}}[\langle\bW^{s,t+1}\!-\!\tilbW^{s,t+1},\!\nabla g(\bW^{s,t})\!-\!\bV_{s,t}\rangle\vert\calF_{s,t}]$ in terms of $\normt{\bW^{s,t}-\bW^{s,0}}^2$ and $\bbE_{\calB_{s,t}}[\normt{\bW^{s,t+1}\!-\!\tilbW^{s,t+1}}^2\vert\calF_{s,t}]$. Then, instead of directly telescoping~\eqref{eq:succ_inner_main}, we construct a surrogate function $\hatf_{s,t}(\bW)\defeq f'(\bW)+\zeta_t\normt{\bW-\bW^{s,0}}^2$, where the sequence $\{\zeta_t\}_{t=0}^m$ is given by the recursion $\zeta_t = (1+1/m)\zeta_{t+1} +(\eta L^2/2)\alpha(n,b)$, where $\zeta_m = 0$. Using this recursive relation, together with the conditions $\eta = 1/(\theta L)$ and $\theta(\theta-1)\ge 2(m+1)\alpha(n,b)$ in Theorem~\ref{thm:main}, we manage to arrive at 
\begin{align}
\bbE_{\calB_{s,t}}[\hatf_{s,t+1}(\bW^{s,t+1})\vert\calF_{s,t}] \le \hatf_{s,t}(\bW^{s,t}) +\eta\left(\eta L-{1}/{2}\right)\normt{\Gamma_{f',\eta}(\bW^{s,t})}^2. \label{eq:toTel_inner_main}
\end{align}
Define $\calB_{s,(t]}\defeq \{\calB_{s,j}\}_{j=0}^t$. We now telescope~\eqref{eq:toTel_inner_main} over $t=0,1,\ldots,m-1$ to obtain
\begin{align}
\hspace{-.25cm}\bbE_{\calB_{s,(m]}}[f'(\bW^{s+1,0})\vert\calF_{s,0}]\! \le\! f'(\bW^{s,0}) + \eta\left(\eta L\!-\!{1}/{2}\right)\sum\nolimits_{t=0}^{m-1}\bbE_{\calB_{s,(t]}}[\normt{\Gamma_{f',\eta}(\bW^{s,t})}^2\vert\calF_{s,0}]. \label{eq:toTel_outer_main}
\end{align}
Finally, we telescope~\eqref{eq:toTel_outer_main} over $s\!\in\!(S\!-\!1]$ 
and use option I to choose the final dictionary $\bW_{\rm final}$.\qed 

{{\bf Proof Sketch of Theorem~\ref{thm:inexact}}}.\ The proof of Theorem~\ref{thm:inexact} is modified from that of Theorem~\ref{thm:main}. Due to the error $\bE_{s,t}$,\!  $\bV_{s,t}$ in~\eqref{eq:succ_inner_main} is replaced by $\hatbV_{s,t}\!\defeq\! \bV_{s,t}\!+\!\bE_{s,t}$. We decouple $\bE_{s,t}$ from $\bV_{s,t}$ through 
$\langle\bW^{s,t+1}-\tilbW^{s,t+1},\nabla g(\bW^{s,t})-\hatbV_{s,t}\rangle\! \le\! \frac{1}{2\eta} \normt{\bW^{s,t+1}-\tilbW^{s,t+1}}^2 + {\eta} \normt{{\nabla g(\bW^{s,t})-\bV_{s,t}}}^2 + \eta\normt{\bE_{s,t}}^2$. We then take expectation on both sides and carefully follow the telescoping procedure (over $t$ and $s$) used in proving Theorem~\ref{thm:main}. After some algebraic manipulations, we arrive at 
\begin{equation}
f^*\le f(\bW^{0,0}) + \eta\left(\eta L-{1}/{2}\right)\sum\nolimits_{s=0}^{S-1}\sum\nolimits_{t=0}^{m-1}\bbE[\normt{\Gamma_{f',\eta}(\bW^{s,t})}^2] + \eta\sum\nolimits_{s=0}^{S-1}\sum\nolimits_{t=0}^{m-1}\norm{\bE_{s,t}}^2.\nn
\end{equation}
Since $\normt{\bE_{s,t}}=(ms+t)^{-(1/2+\tau)}\,(\tau\!>\!0)$, for any $S\ge 1$, $\sum_{s=0}^{S-1}\sum_{t=0}^{m-1}\normt{\bE_{s,t}}^2 < \infty$. Therefore by using option I to choose $\bW_{\rm final}$, we complete the proof. \qed 

{\bf Proof of Corollary~\ref{cor:complexity}}.\ 
By Theorem~\ref{thm:main}, we know that $O(1/\epsilon)$ {\em inner iterations} are needed for $\bbE[\normt{\Gamma_{\!f'\!,\eta}(\bW_{\rm final})}^2]\!\le\! \epsilon$. 
By evenly distributing the $n$ data samples drawn at the start of each outer iteration into $m$ inner iterations, we see that each inner iteration takes in $O(n/m+b)$ samples. Thus the sample complexity is $O((n/m+b)/\epsilon)$. Based on it, we can also derive the computational complexity to be $O((n/m+b)dk/\epsilon)$. 
This is because both subproblem~\eqref{eq:loss_2} (in lines~\ref{line:solve_hr1}, \ref{line:solve_hr2} and~\ref{line:solve_hr3}) and the proximal operator of $\eta\psi+\delta_\calC$ (in line~\ref{line:prox_W}) can be solved or evaluated in $O(dk)$  arithmetic operations (see Section~\ref{sec:algo}).  
From the condition $\theta(\theta-1)\ge 2m(m+1)\alpha(n,b)$ in Theorem~\ref{thm:main}, we see that $m=O(\sqrt{nb/(n-b)})=O(\sqrt{b})$.
Thus the sample and computational complexities become $O((n/\sqrt{b}+b)/\epsilon)$ and $O((n/\sqrt{b}+b)dk/\epsilon)$ respectively. To optimize both costs, it is necessary that $b=\Theta(n/\sqrt{b})$, which amounts to $b=\Theta(n^{2/3})$. (Note that this also implies $m=\Theta(n^{1/3})$.) 
\qed

\section{Numerical Experiments}\label{sec:numerical}

\subsection{Tested SMF Formulations and Datasets}
We tested the performance of our proposed framework (Algorithm~\ref{algo:SMF_VR}) on four representative SMF formulations in Section~\ref{sec:prob_form}, including ODL, ONMF, ORPCA and ORNMF. 
Among them, ODL and ONMF are non-robust SMF formulations whereas ORPCA and ORNMF are robust ones, i.e., they explicitly model outliers. Therefore, for ODL and ONMF, we tested their algorithms on the {\tt CBCL}~\cite{Lee_99} and {\tt MNIST}~\cite{LeCun_98} datasets; while for ORPCA and ORNMF, we tested their algorithms on the {\tt Synth} and {\tt YaleB}~\cite{YaleB_01} datasets. The {\tt CBCL} and {\tt MNIST} datasets are commonly used in testing (non-robust) matrix factorization algorithms, e.g., \cite{Lee_99}. 
The {\tt YaleB} dataset consists of face images taken under various lighting conditions. The shadow and gloss in these images caused by imperfect illumination can be regarded as outliers. 
The {\tt Synth} dataset was generated synthetically in a similar fashion to those in~\cite{Feng_13,Shen_14b,Zhao_17}. Specifically, we first generated a $d\!\times\! k'$ matrix $\barbW$ and a $k'\!\times\! n$ matrix $\barbH$, where $d\!=\!400$, $n\!=\!1\!\times\!10^5$ and $k'\!=\!10$. The entries of $\barbW$ and $\barbH$ were drawn \iid from a normal distribution with  mean 0.5 and variance $1/\!\sqrt{k'}$. We then generated a $d\!\times\! n$ outlier matrix $\barbR$ by first uniformly randomly selecting $\floor{(1\!-\!\rho_s) dn}$ of its  entries to be zero, where $\rho_s$ denotes the outlier density. The remaining entries were then generated \iid from a uniform distribution with support $[-M,\!M]$. We set $\rho_s\!=\!0.1$ and $M\!=\!1000$. 
Finally $\bV\!=\!\barbW\,\barbH\!+\!\barbR$. 

\subsection{Benchmarking Frameworks and Choice of Parameters}
 For each tested SMF formulation, we compared our variance-reduced SMF framework (denoted as VR) against another two optimization frameworks commonly used in previous works on SMF, namely SMM and SGD. For completeness, the pseudo-codes for these two methods are shown in Algorithms~\ref{algo:SMF_SMM} and \ref{algo:SMF_SGD}  respectively. We next describe the parameter setting in our method. 
From our analysis in Corollary~\ref{cor:complexity}, we set the mini-batch size $b = c_1n^{2/3}$ and the number of inner iterations $m=c_2n^{1/3}$, where $c_1=0.2$ and $c_2=0.5$. We set the number of outer iterations $S\!=\!10$ and the parameter $\tau$ in Theorem~\ref{thm:inexact} to $1\!\times\!10^{-3}$. For each tested SMF formulation on each dataset, we chose the step size $\eta$ such that our method yielded the best (empirical) convergence rate. The plots of objective values 
resulting from other step sizes are shown in Figure~\ref{fig:etas}. For the latent dimension $k$, following the convention in~\cite{Lee_99,Mairal_10,Zhao_17}, we set it to 49.  This parameter can also be chosen from domain knowledge or a Bayesian approach, e.g., \cite{Bishop_98}. Since in practice we found the performance of all the three frameworks (VR, SMM and SGD) was not sensitive to $k$, we fixed it for simplicity. For both SMM and SGD, we used the same mini-batch size $b$ as in VR. The step sizes $\{\gamma_t\}_{t\ge 0}$ in SGD are chosen in the classical way~\cite{Borkar_08}, i.e., $\gamma_t=\beta/(bt+\beta')$. Similar to VR, we chose $\beta,\beta'>0$ such that SGD achieved the best empirical convergence rate. For simplicity, we initialized the dictionary $\bW$ such that all its entries are equal to one. Finally, following~\cite{Mairal_10,Guan_12,Feng_13,Zhao_17}, we set the regularization weight $\lambda$ in ODL, ONMF and ORNMF to $1/\sqrt{d}$. In addition, we set $\lambda_1=\lambda_2=1/\sqrt{d}$ in ORPCA.

\subsection{Plots of Objective Values}
For each SMF formulation and each dataset, since we focus on convergence to a stationary point, \vspace{-.05cm}we first run a {\em batch} (deterministic) gradient-based matrix factorization  algorithm (e.g.,~\cite{Lin_07a}) to  estimate a stationary point $\hatbW$. \vspace{-.05cm}This resembles the way to use batch methods to estimate a global optimum in the literature of {\em stochastic convex optimization}~\cite{Johnson_13,Defazio_14}. \vspace{-.05cm} Based on $\hatbW$, we plot the log-suboptimality of the objective value, i.e., $\log(f(\bW)-f(\hatbW))$ versus both running time  and number of data passes, for  VR, SMM and SGD respectively (see Figure~\ref{fig:obj}).\footnote{All the algorithms are implemented in Matlab\textsuperscript{\textregistered} 8.6 and run on a machine with a 3.6-GHz processor.} 
We plot objective values versus number of data passes so that the results in Figure~\ref{fig:obj} are agnostic to the actual implementation of the algorithms. 
From Figure~\ref{fig:obj}, in terms of both time and number of data passes, we observe that for all the SMF algorithms and datasets, our variance-reduced framework (VR) not only converges faster than SMM and SGD, \vspace{-.05cm}but also find a more accurate approximate of the stationary point $\hatbW$. Explanations and discussions of these observations are deferred to Section~\ref{sec:exp_discussion}. 

\subsection{Subspace Recovery by ORPCA}
 We also considered the subspace recovery task using ORPCA~\cite{Feng_13} on the {\tt Synth} dataset.  
The ground-truth subspace $\calU$ is given by the column space of the (ground-truth) dictionary $\barbW$ in generating the {\tt Synth} dataset (see Section~5.1). Accordingly, the estimated subspace $\calU'$ at any time instant is given by the column space of our learned dictionary. The similarity between $\calU'$ and $\calU$ is measured by the expressed variance (EV) (see the definition in~\cite{Xu_13}). A larger value of EV indicates a higher similarity, hence a better subspace recovery result. (A unit EV indicates perfect recovery, i.e., $\calU'\!=\!\calU$.) We employed VR, SMM and SGD on the ORPCA problem to recover the subspace $\calU$. In addition to the original {\tt Synth} dataset with outlier density $\rho_s\!=\!0.1$, we generated another one with $\rho_s\!=\!0.3$. From the results in Figure~\ref{fig:subspace_rec}, we observe that VR 
consistently outperforms the other two frameworks. Specifically, compared to SMM and SGD, VR not only recovers $\calU$ faster, but also recovers a subspace that is closer to $\calU$ (in terms of EV). These observations are indeed consistent with those in Section~5.3, and are explained in Section~\ref{sec:exp_discussion}. 

\begin{figure}[t]
\subfloat{\includegraphics[width=.26\columnwidth,height=.185\columnwidth]{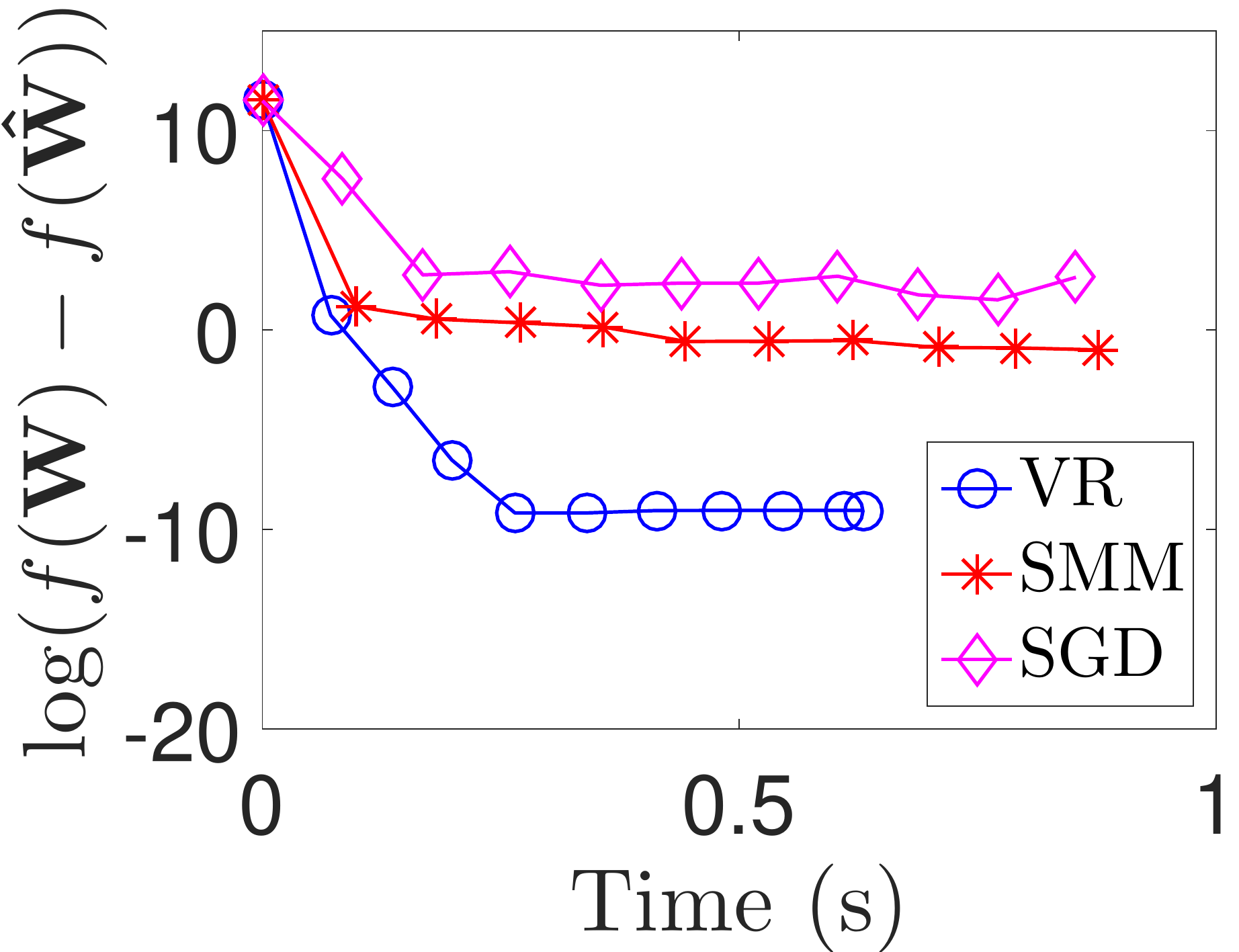}}\hfill
\subfloat{\includegraphics[width=.24\columnwidth,height=.18\columnwidth]{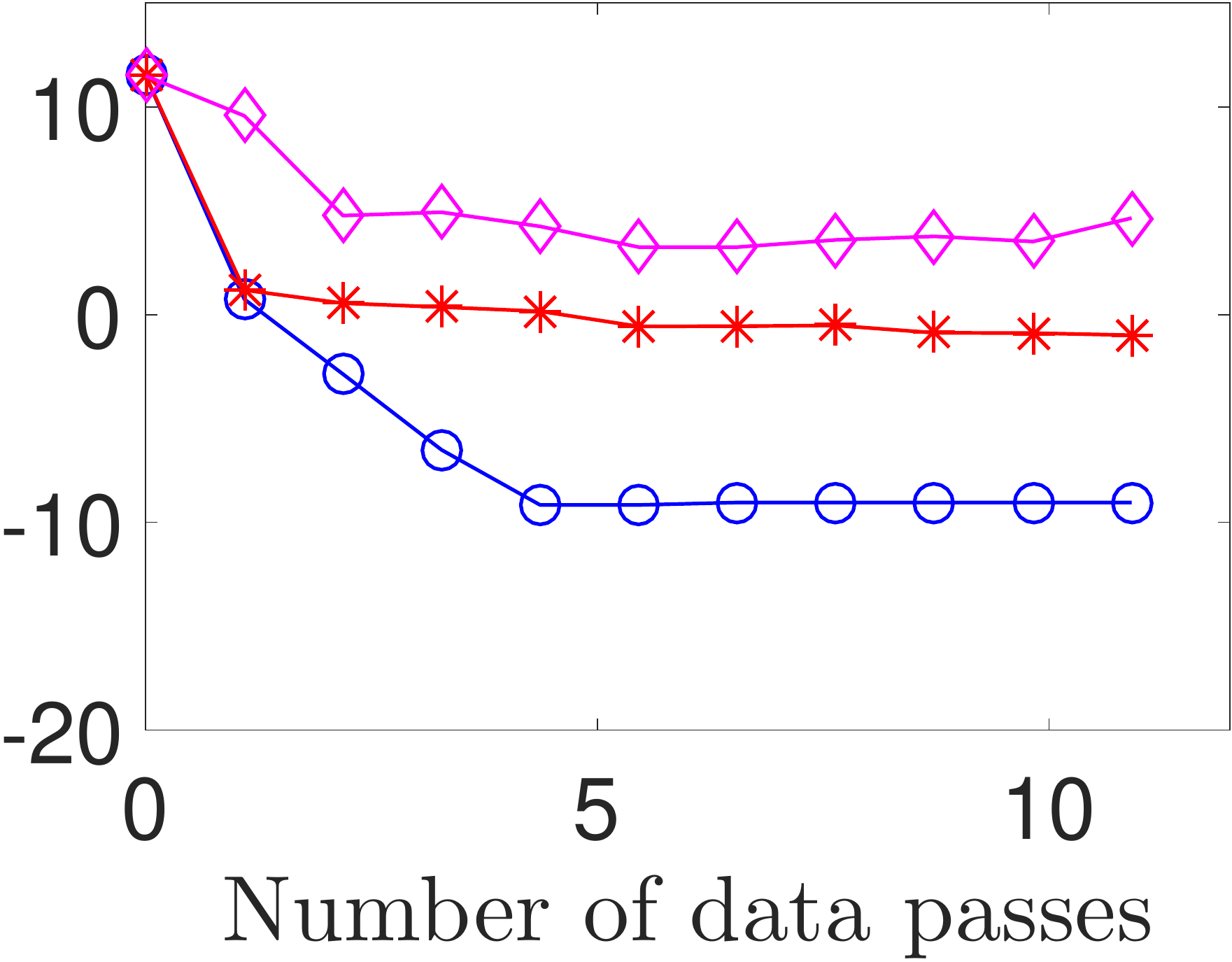}}\hfill
\subfloat{\includegraphics[width=.24\columnwidth,height=.18\columnwidth]{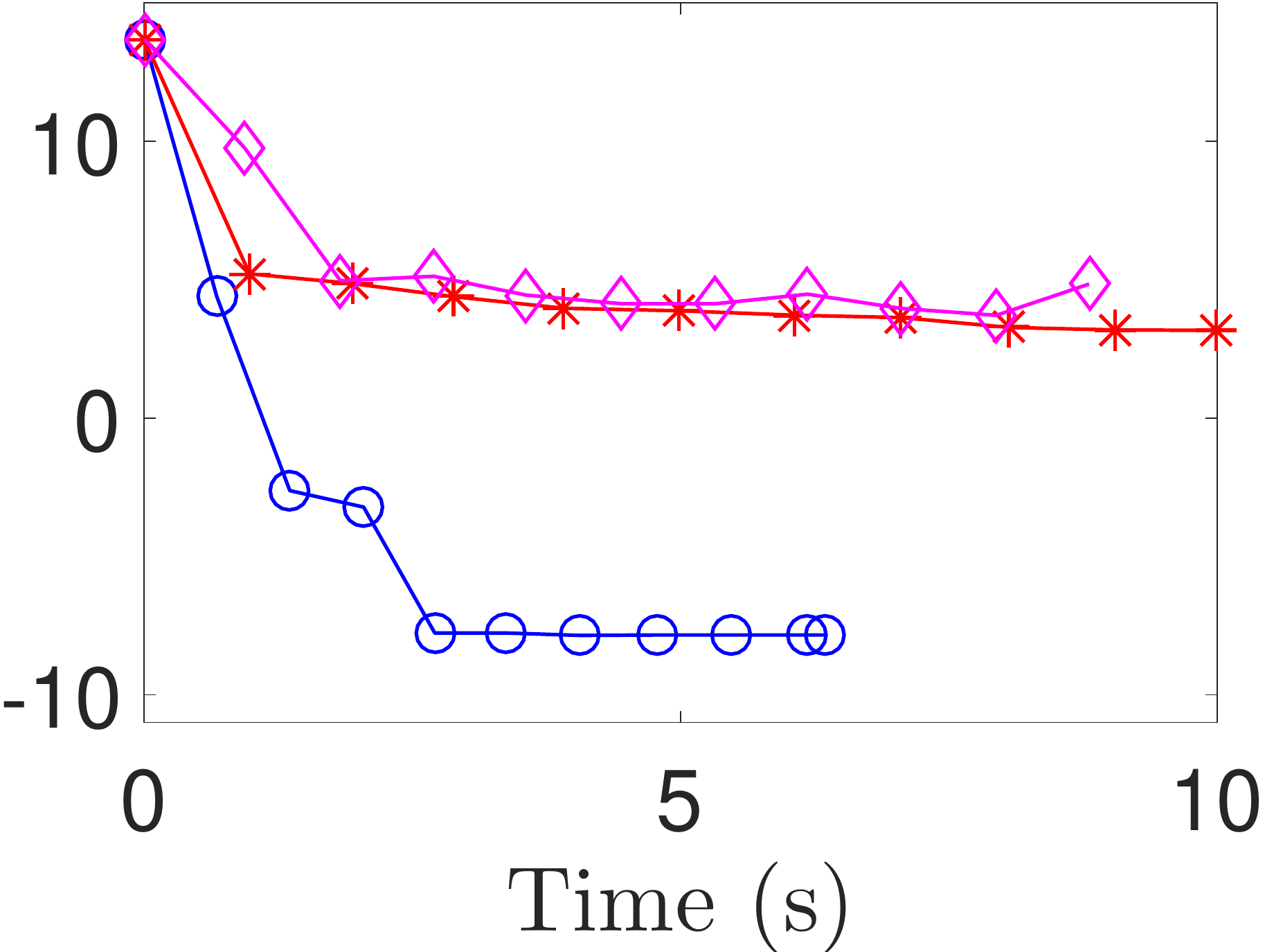}}\hfill
\subfloat{\includegraphics[width=.24\columnwidth,height=.18\columnwidth]{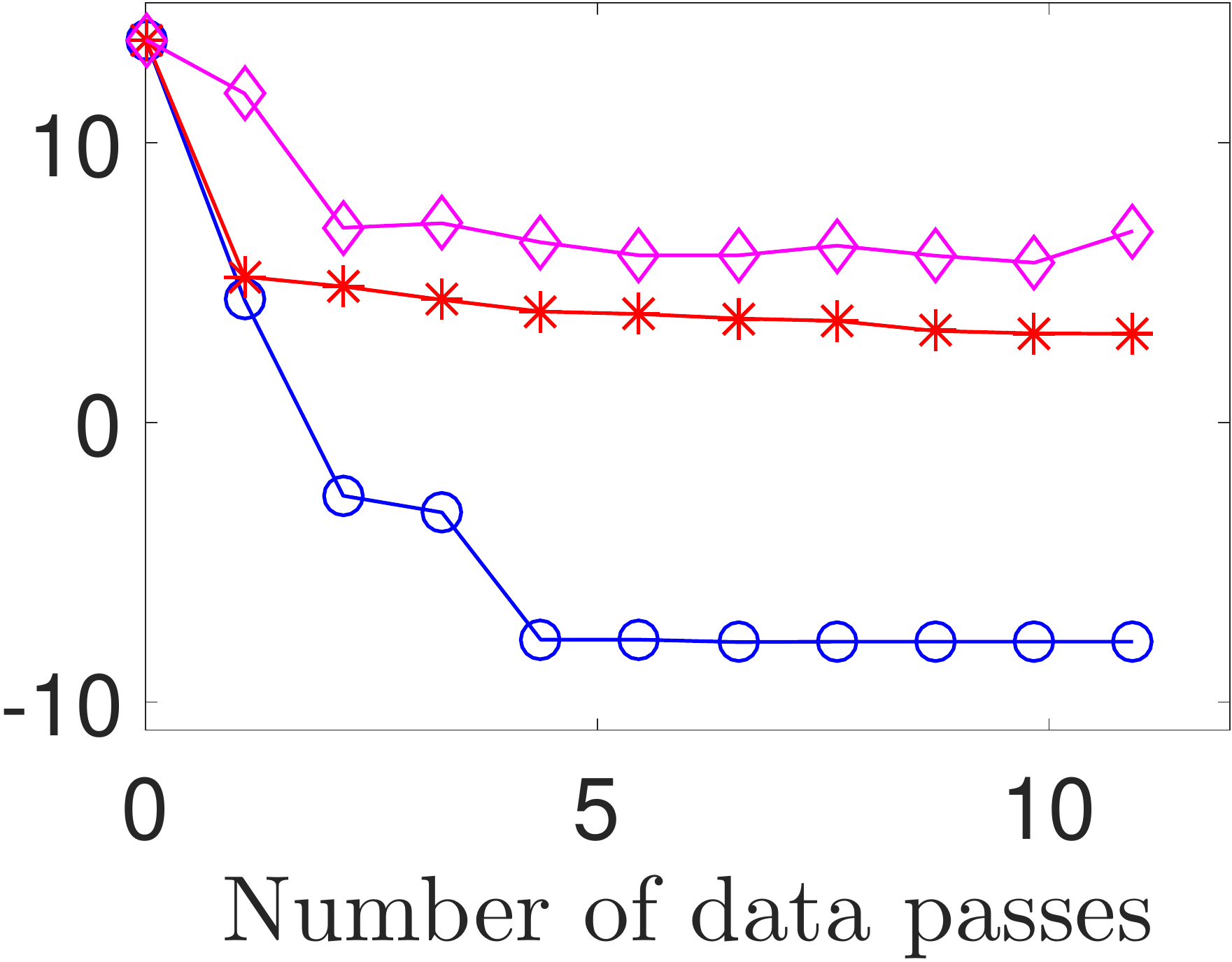}}\\[-.1cm]
\hspace*{3cm} (a) ODL ({\tt CBCL}) \hspace*{4cm} (b) ODL ({\tt MNIST})\\[-.25cm]
\subfloat{\includegraphics[width=.26\columnwidth,height=.185\columnwidth]{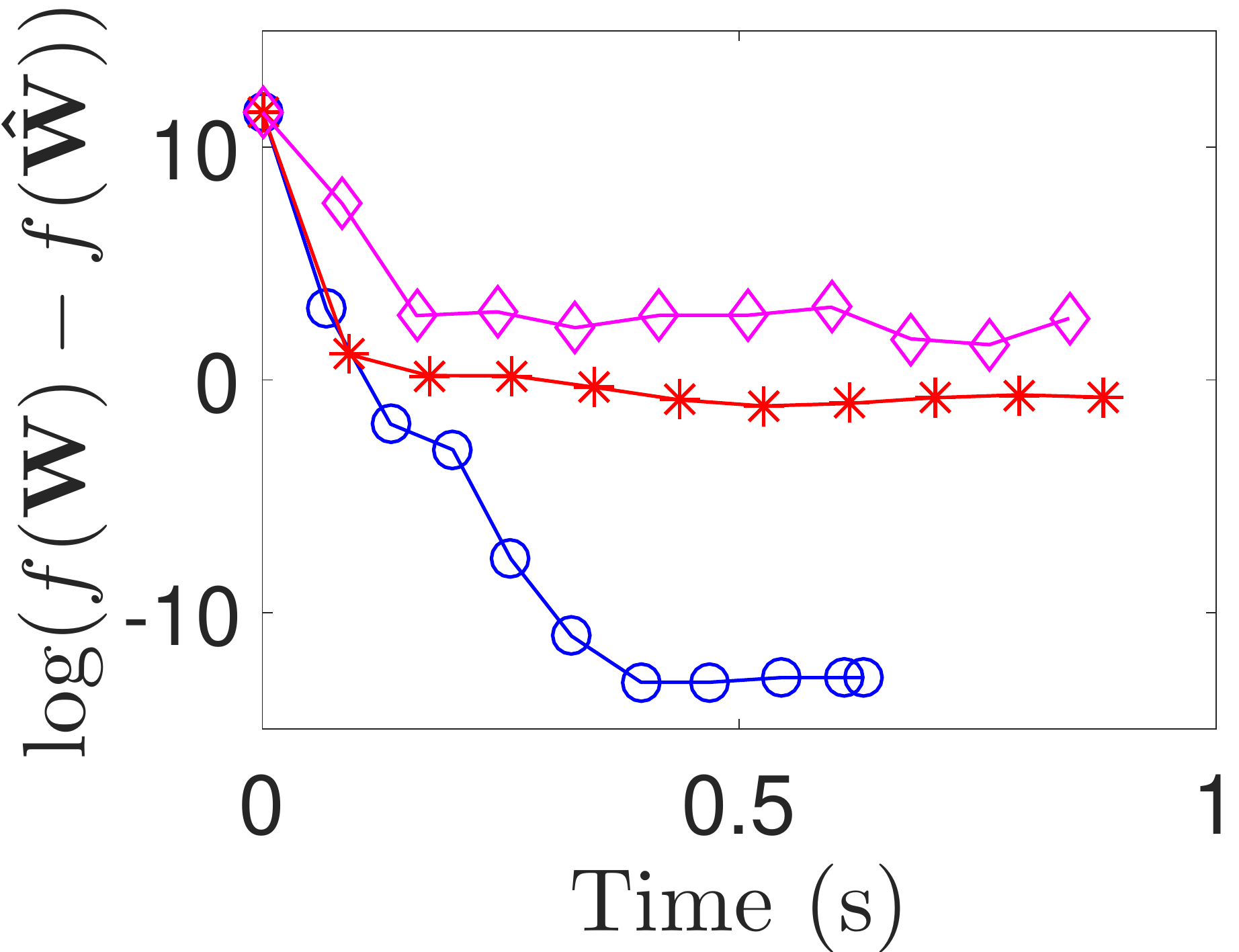}}\hfill
\subfloat{\includegraphics[width=.24\columnwidth,height=.18\columnwidth]{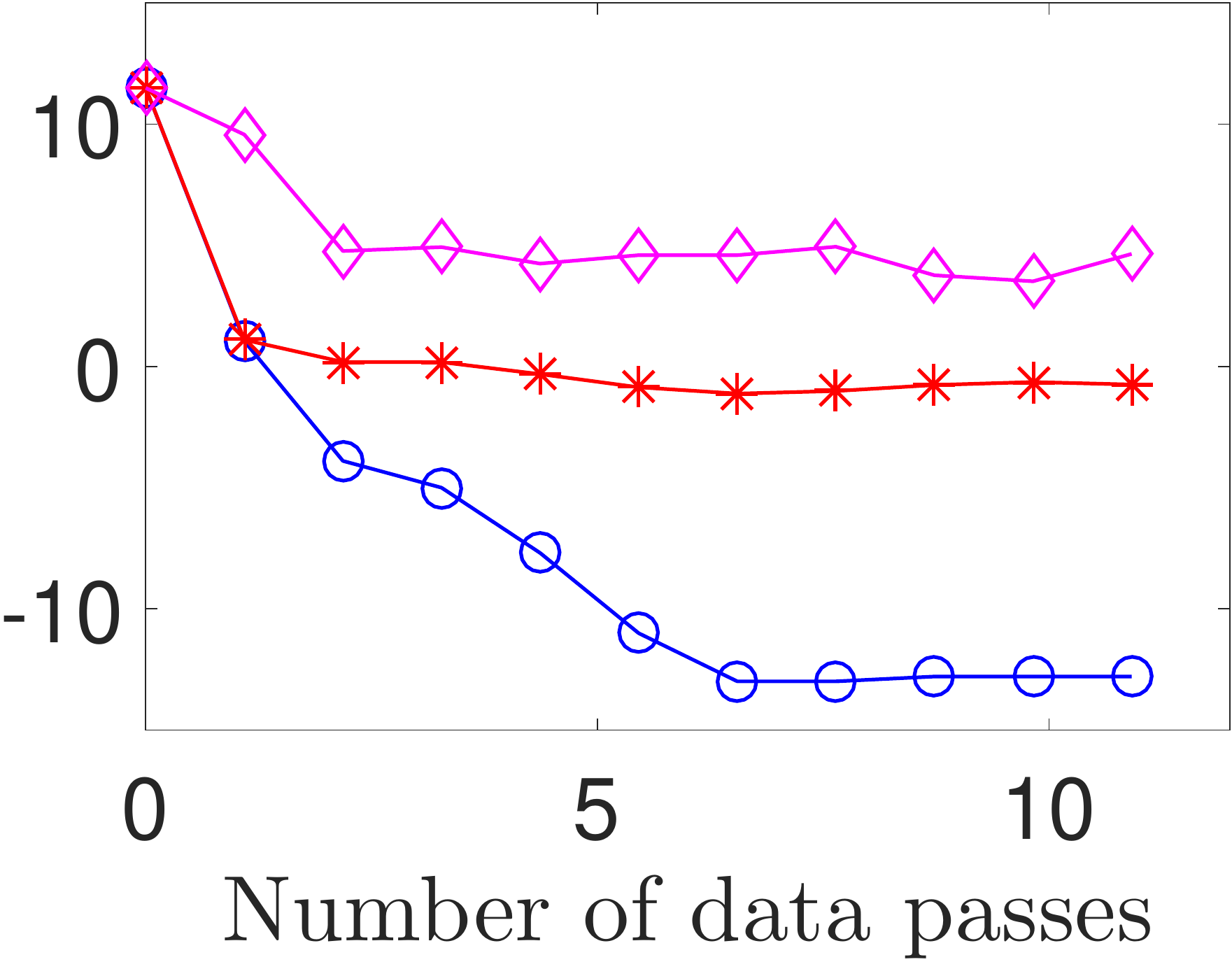}}\hfill
\subfloat{\includegraphics[width=.24\columnwidth,height=.18\columnwidth]{ODL_mnist_diffAlgoTime.eps}}\hfill
\subfloat{\includegraphics[width=.24\columnwidth,height=.18\columnwidth]{ODL_mnist_diffAlgo2.eps}}\\[-.1cm]
\hspace*{2.8cm} (c) ONMF ({\tt CBCL}) \hspace*{4cm} (d) ONMF ({\tt MNIST})\\[-.25cm]
\subfloat{\includegraphics[width=.28\columnwidth,height=.185\columnwidth]{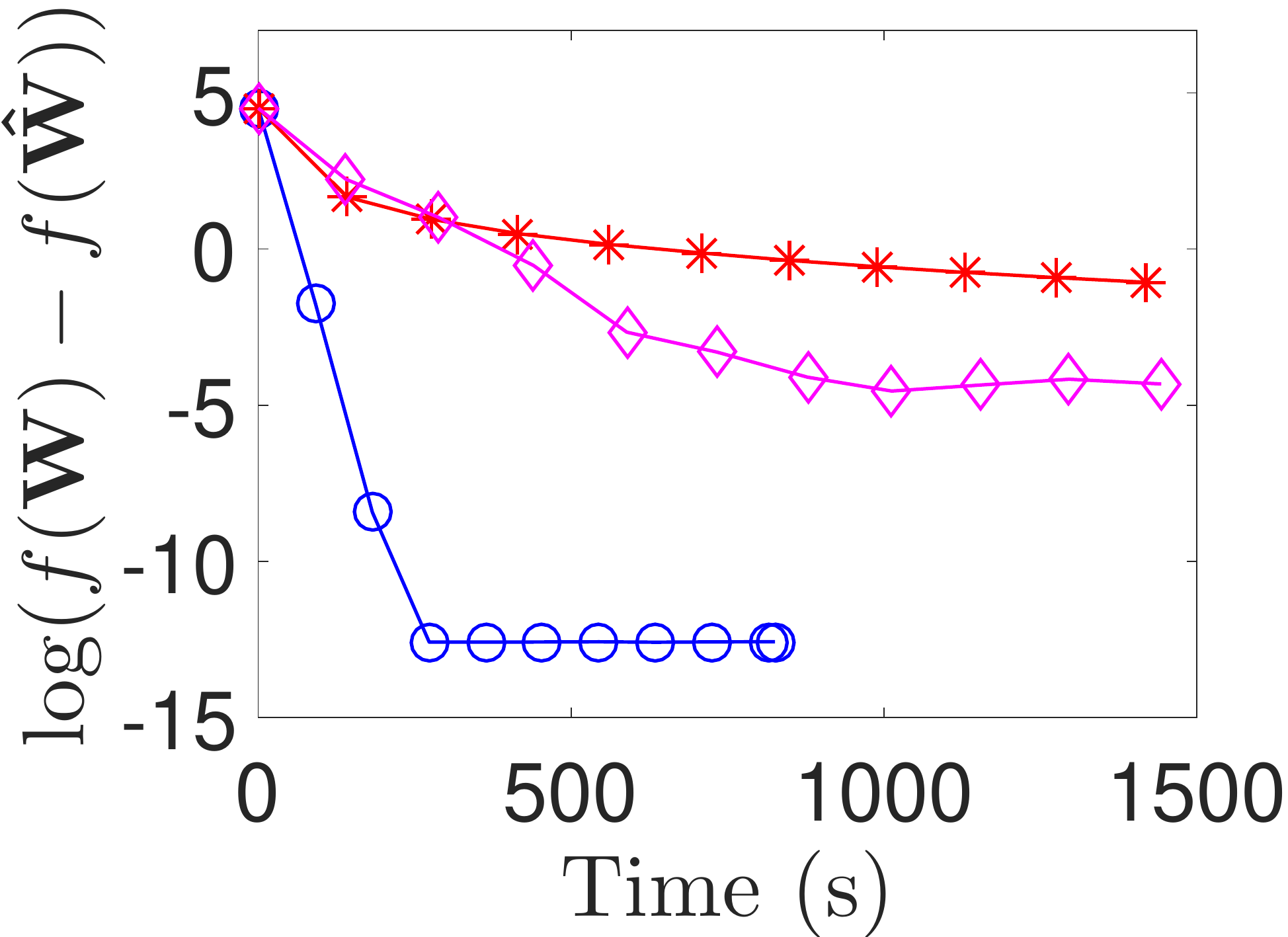}}\hfill
\subfloat{\includegraphics[width=.22\columnwidth,height=.18\columnwidth]{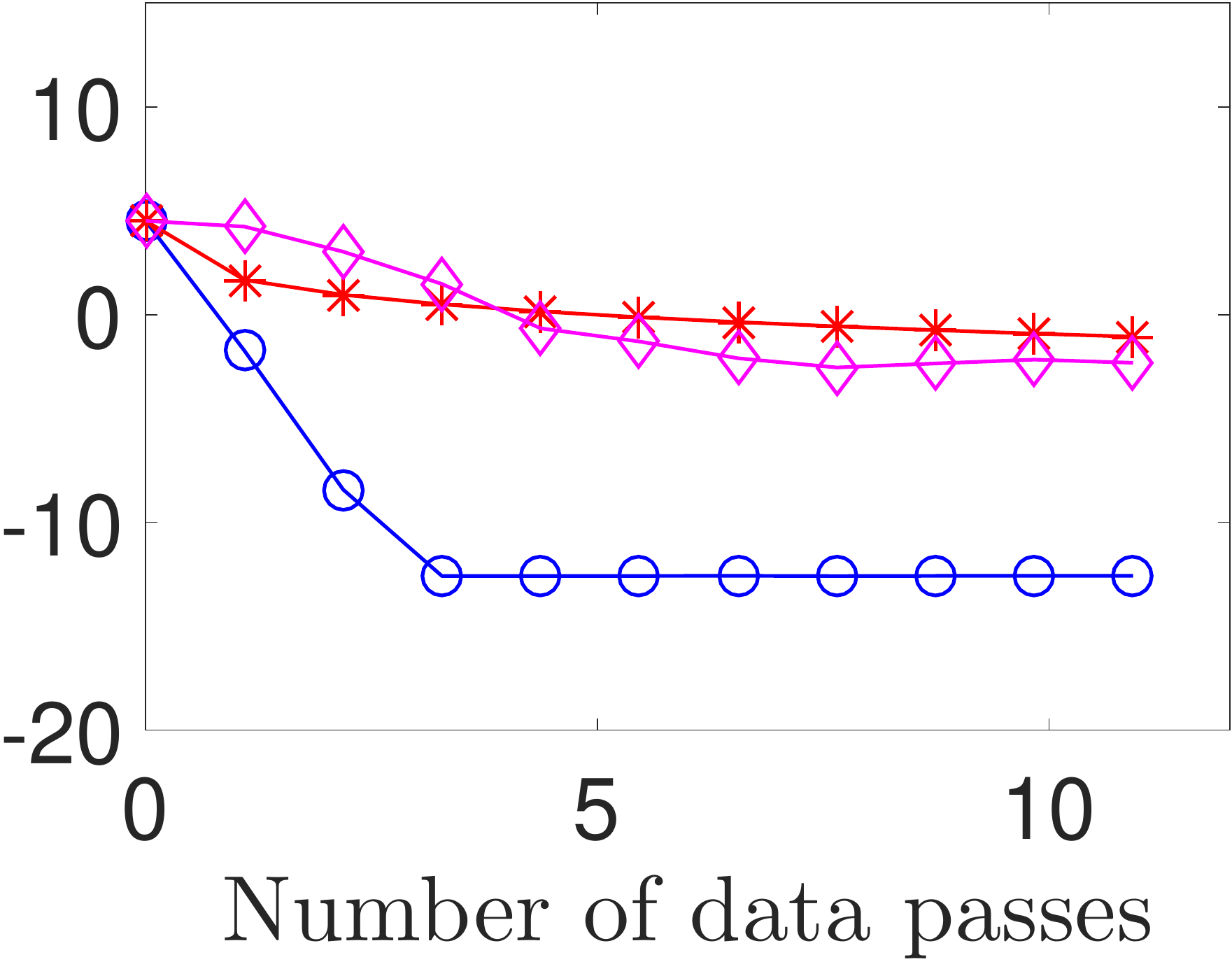}}\hfill
\subfloat{\includegraphics[width=.24\columnwidth,height=.188\columnwidth]{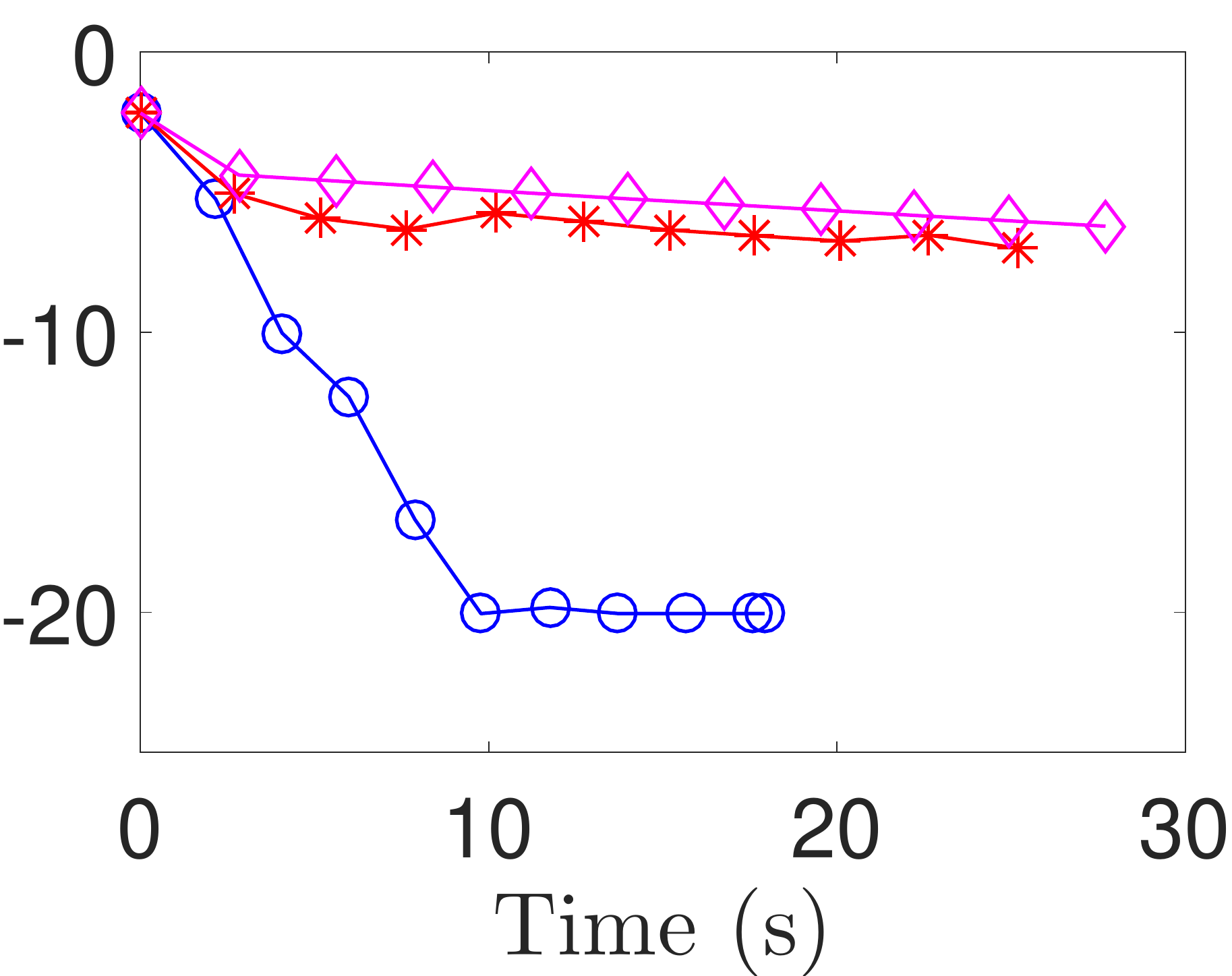}}\hfill
\subfloat{\includegraphics[width=.24\columnwidth,height=.188\columnwidth]{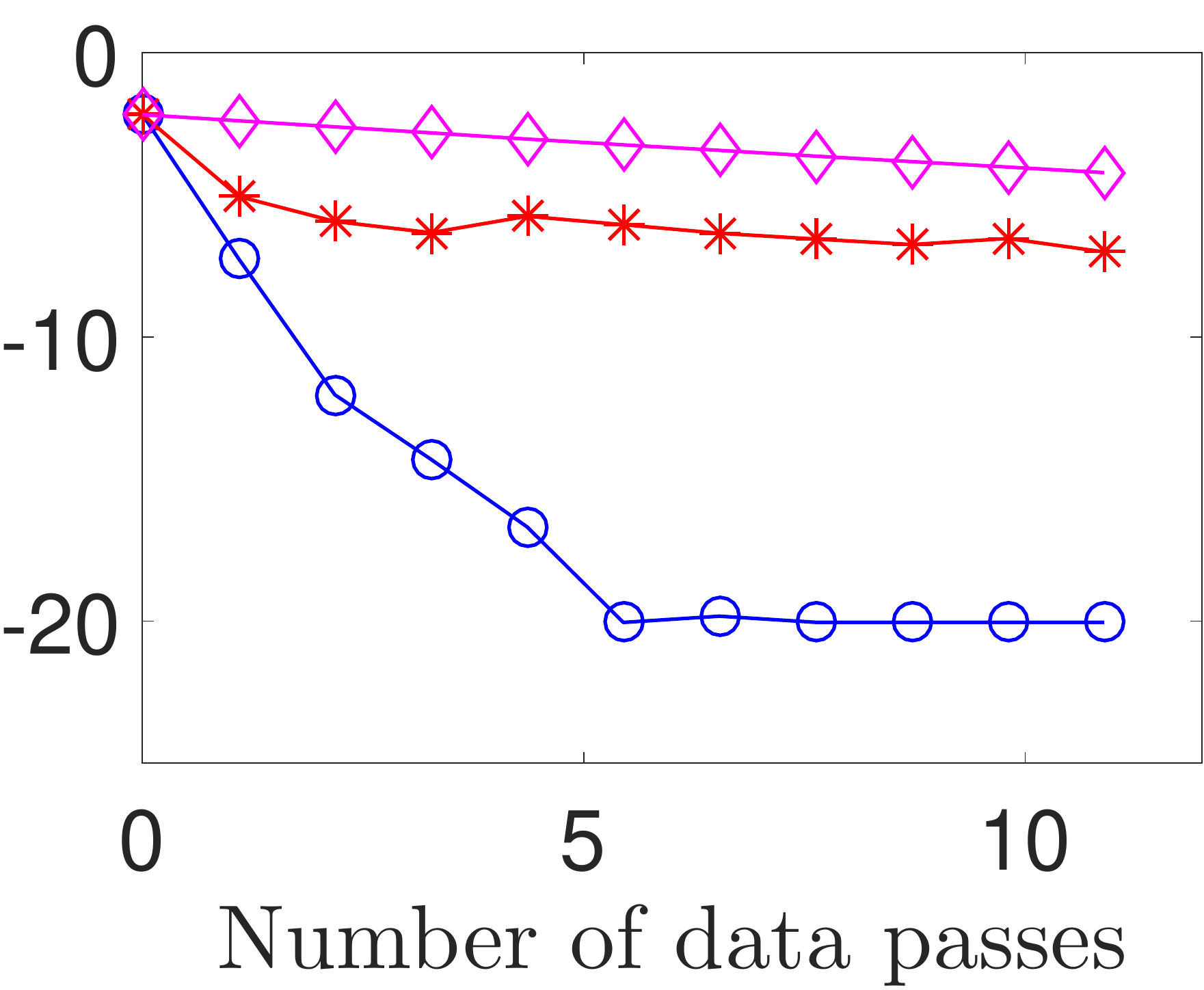}}\\[-.1cm]
\hspace*{2.8cm} (e) ORPCA ({\tt Synth}) \hspace*{3.6cm} (f) ORPCA ({\tt YaleB})\\[-.25cm]
\subfloat{\includegraphics[width=.26\columnwidth,height=.185\columnwidth]{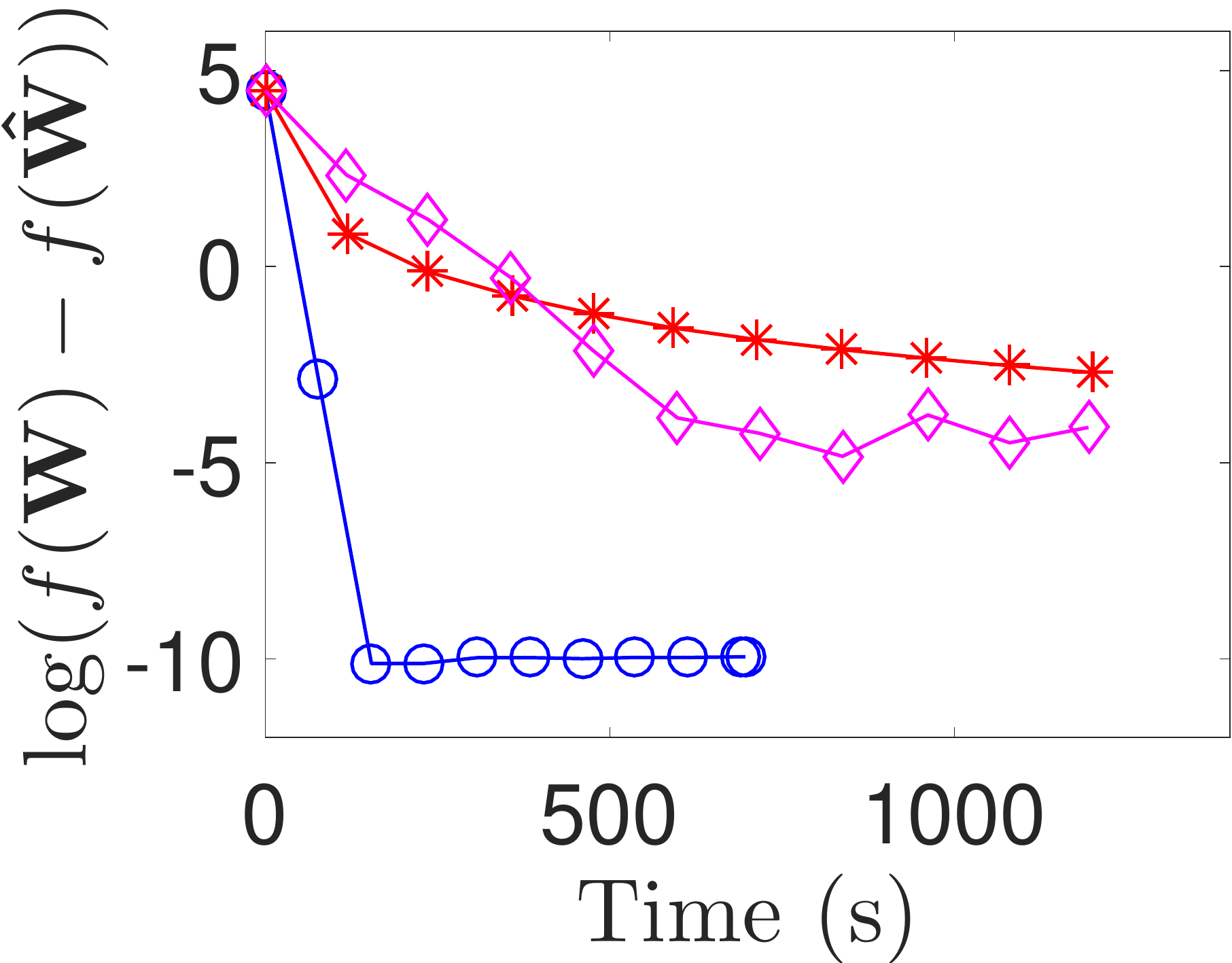}}\hfill
\subfloat{\includegraphics[width=.24\columnwidth,height=.18\columnwidth]{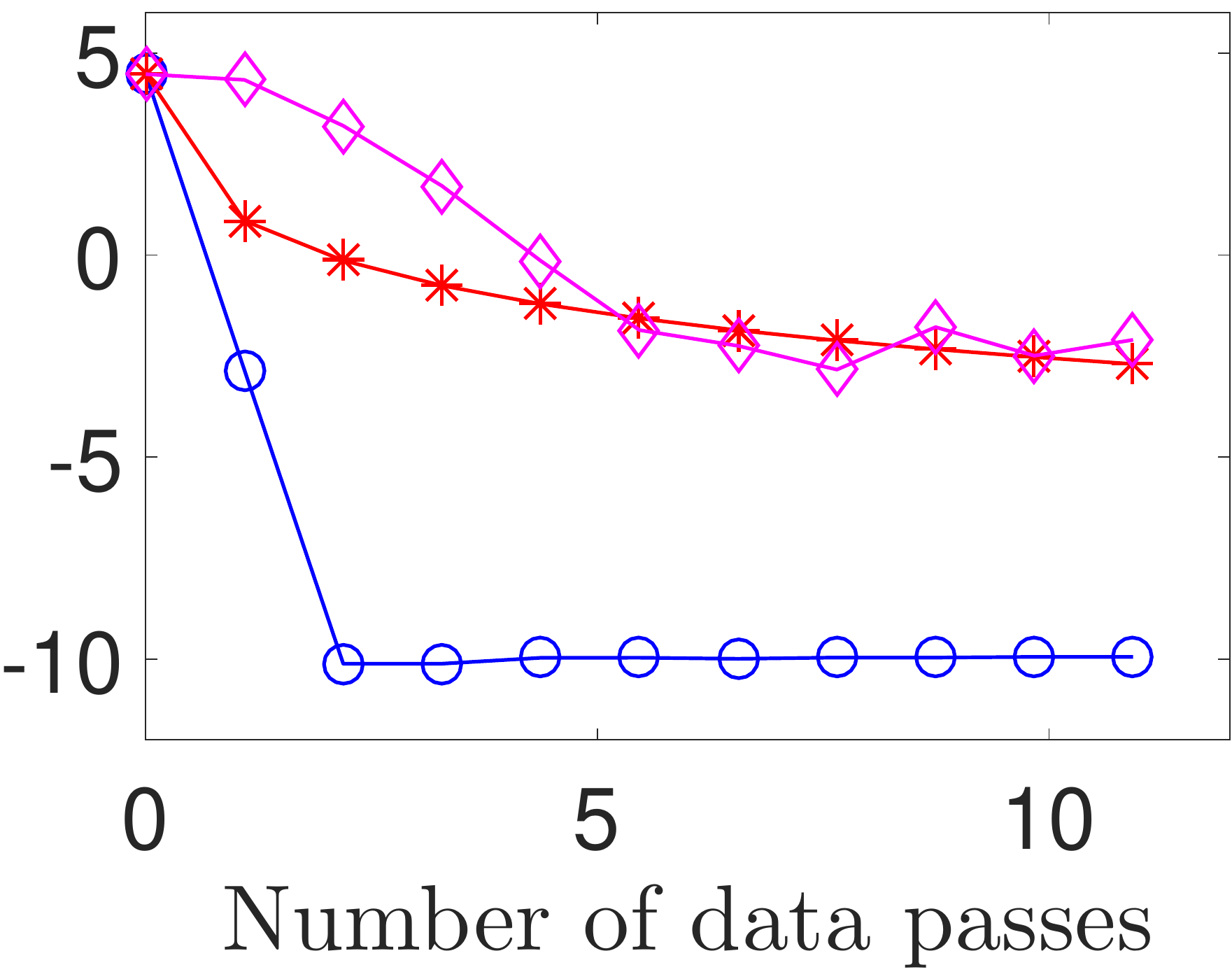}}\hfill
\subfloat{\includegraphics[width=.24\columnwidth,height=.188\columnwidth]{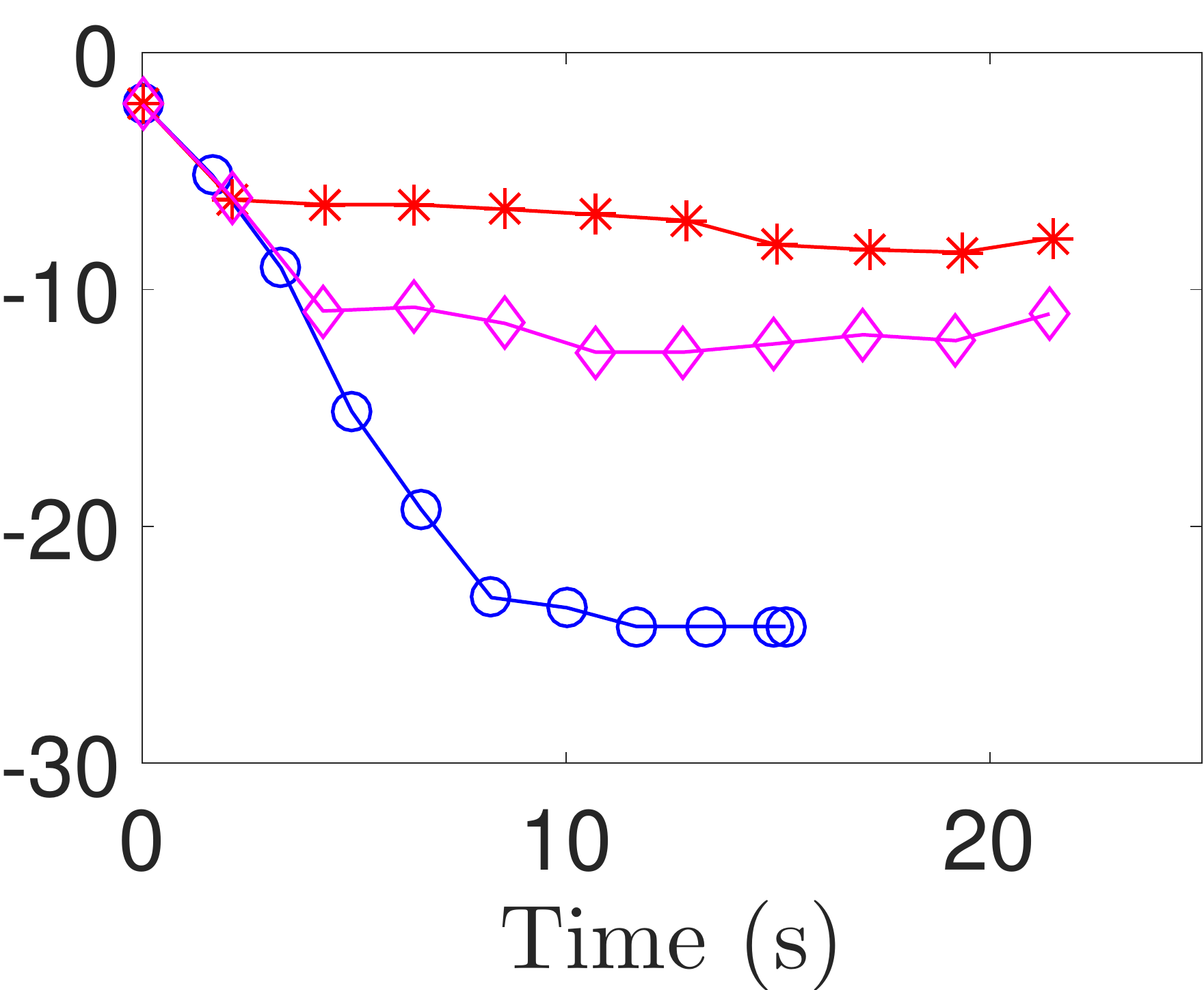}}\hfill
\subfloat{\includegraphics[width=.24\columnwidth,height=.188\columnwidth]{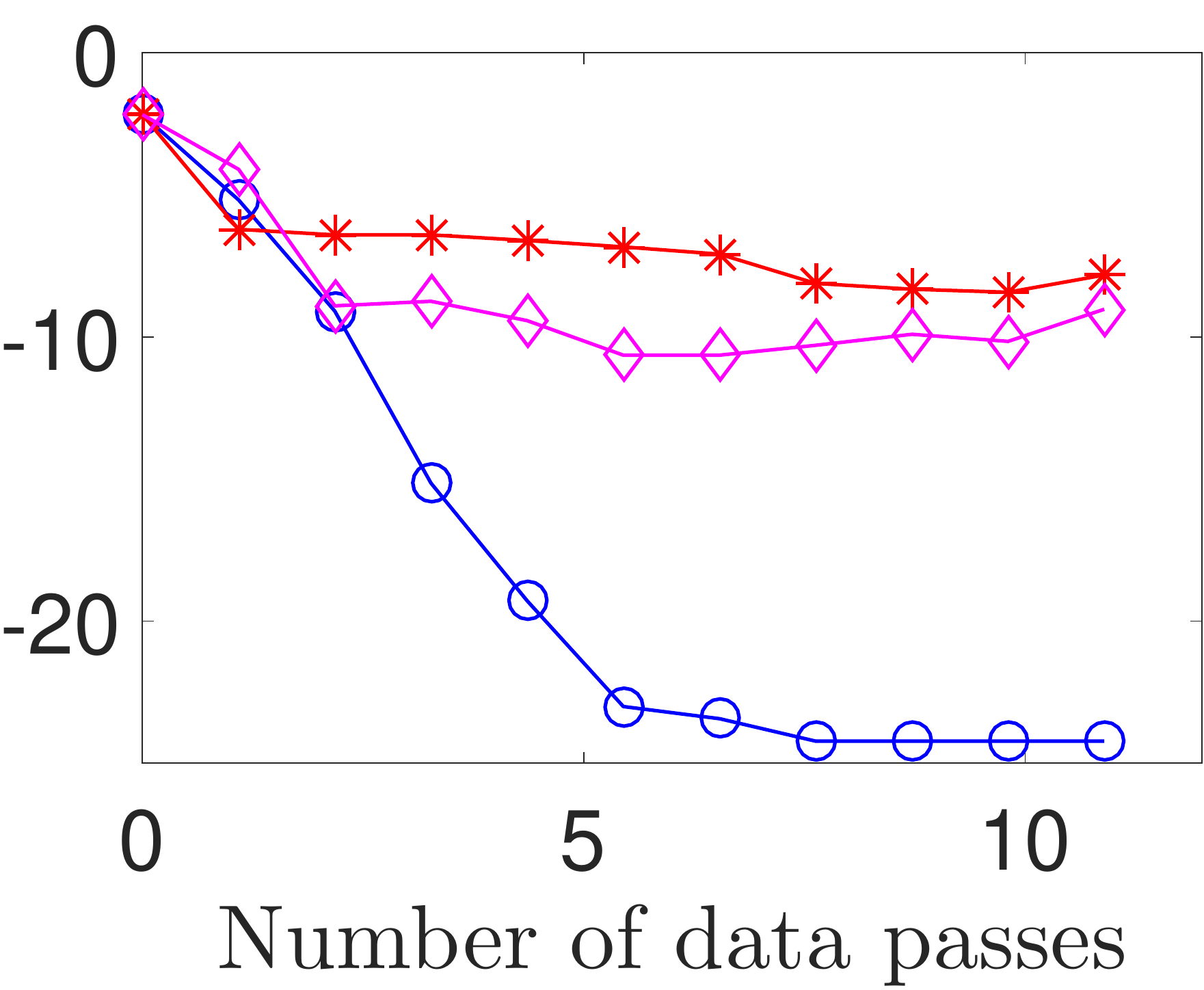}}\\[-.1cm]
\hspace*{2.8cm} (g) ORPCA ({\tt Synth}) \hspace*{3.6cm} (h) ORPCA ({\tt YaleB})
\caption{Plots of objective values produced by the three optimization framworks (VR, SMM and SGD) versus {\em number of data passes} and {\em time} on the four SMF algorithms (ODL, ONMF, ORPCA and ORNMF)
 and four datasets ({\tt CBCL}, {\tt MNIST}, {\tt Synth} and {\tt YaleB}).
 \label{fig:obj}}
\end{figure}

\begin{figure}[h]
\subfloat{\includegraphics[width=.26\columnwidth,height=.19\columnwidth]{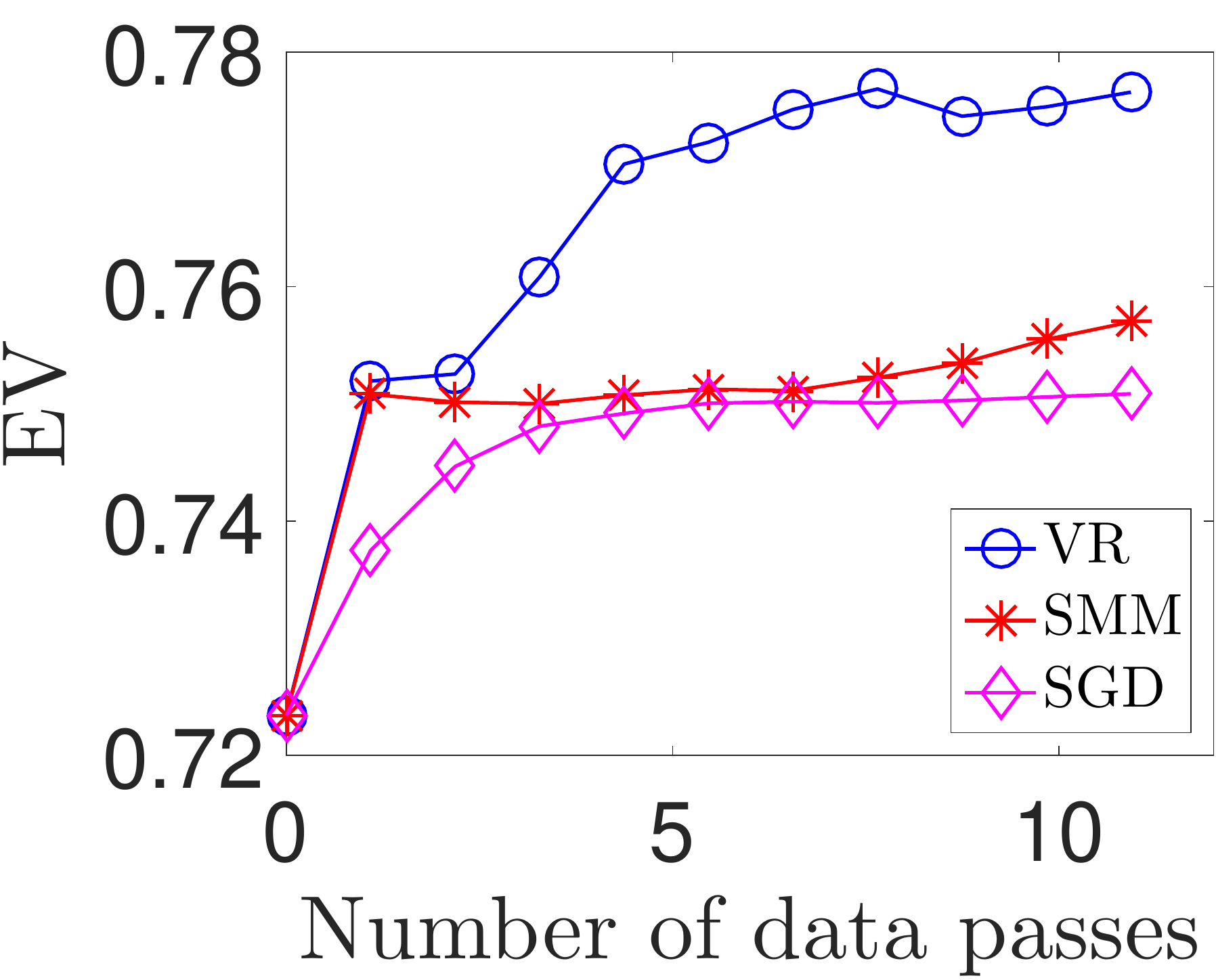}}\hfill
\subfloat{\includegraphics[width=.25\columnwidth,height=.19\columnwidth]{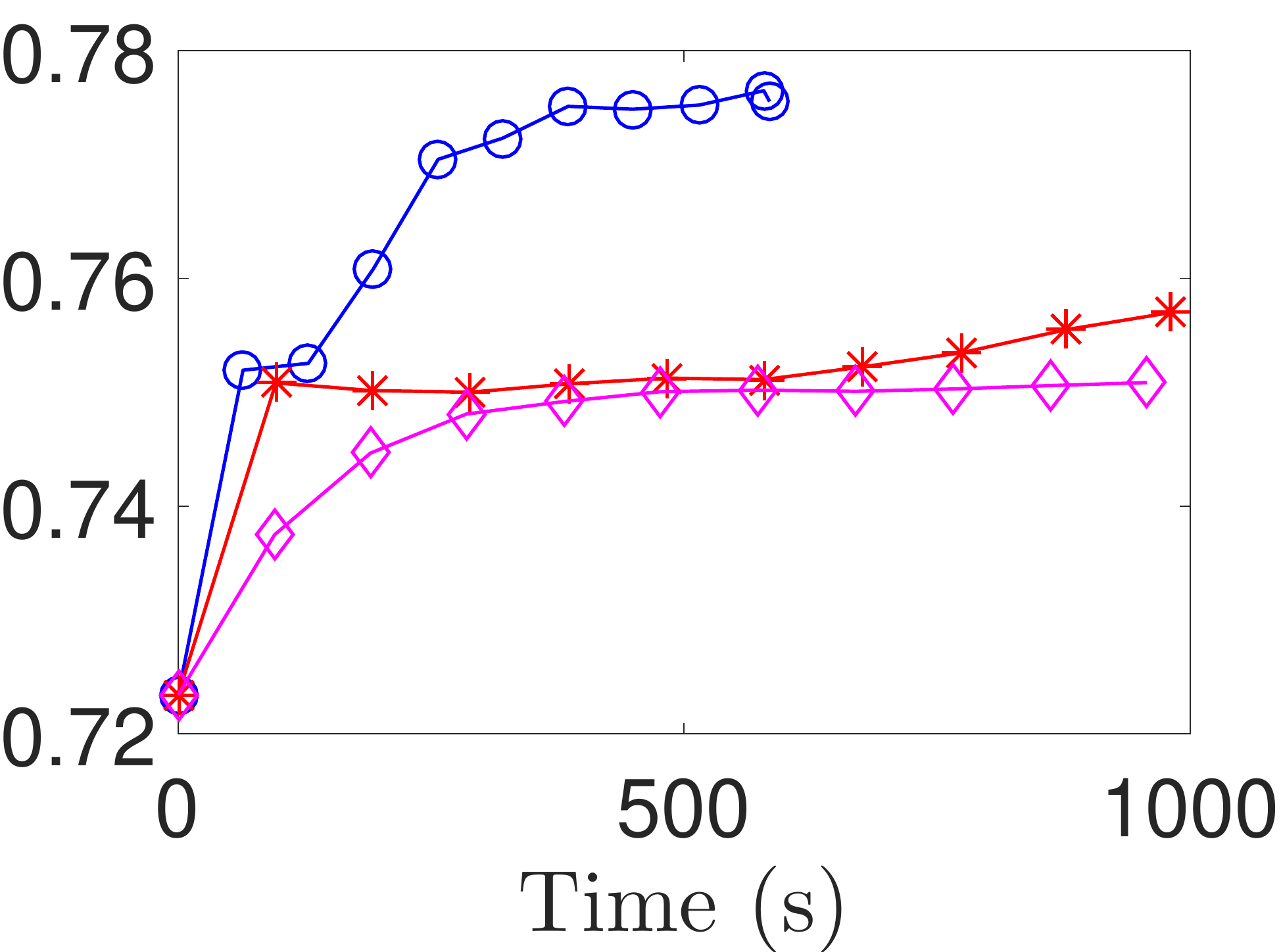}}\hfill
\subfloat{\includegraphics[width=.24\columnwidth,height=.19\columnwidth]{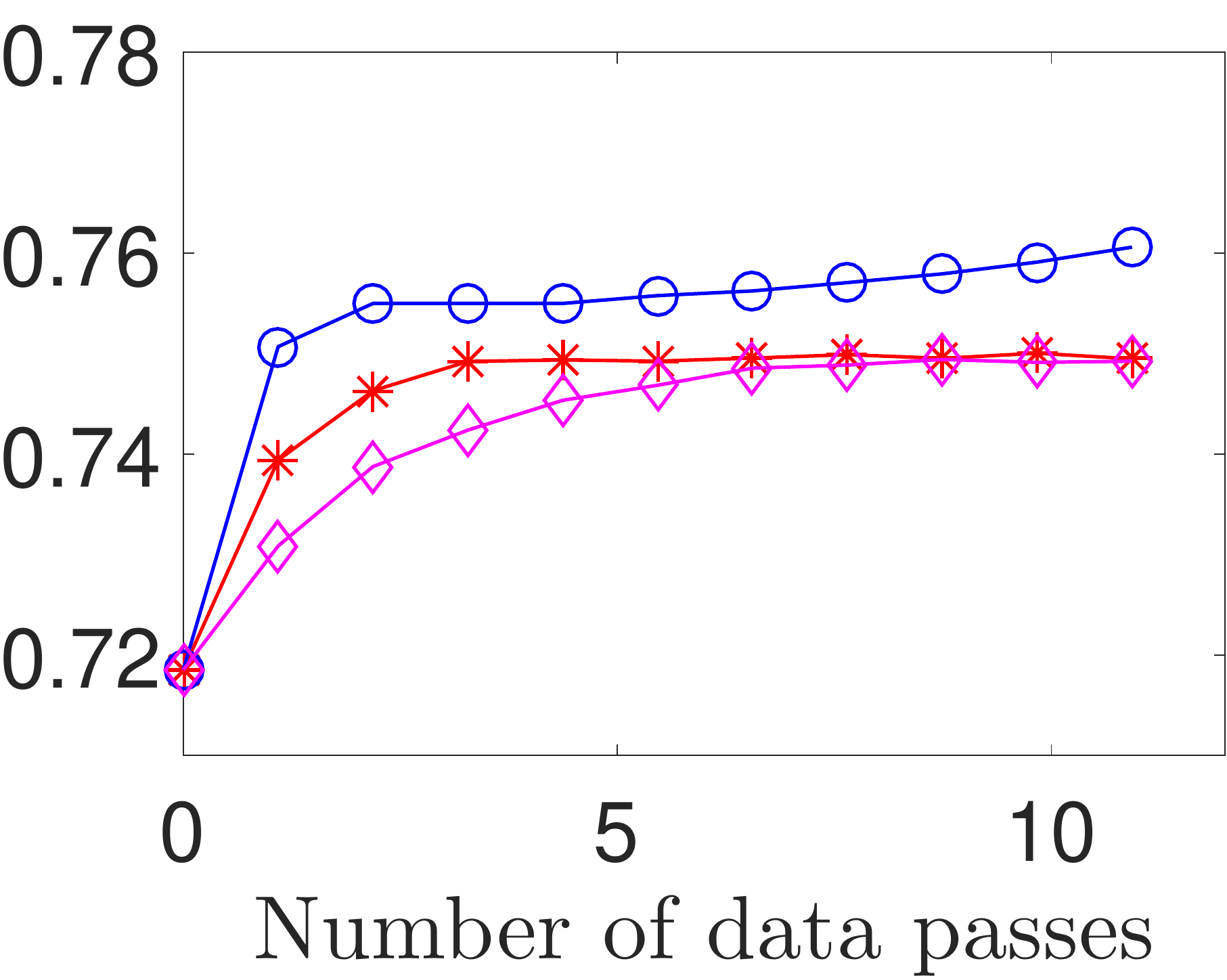}}\hfill
\subfloat{\includegraphics[width=.24\columnwidth,height=.19\columnwidth]{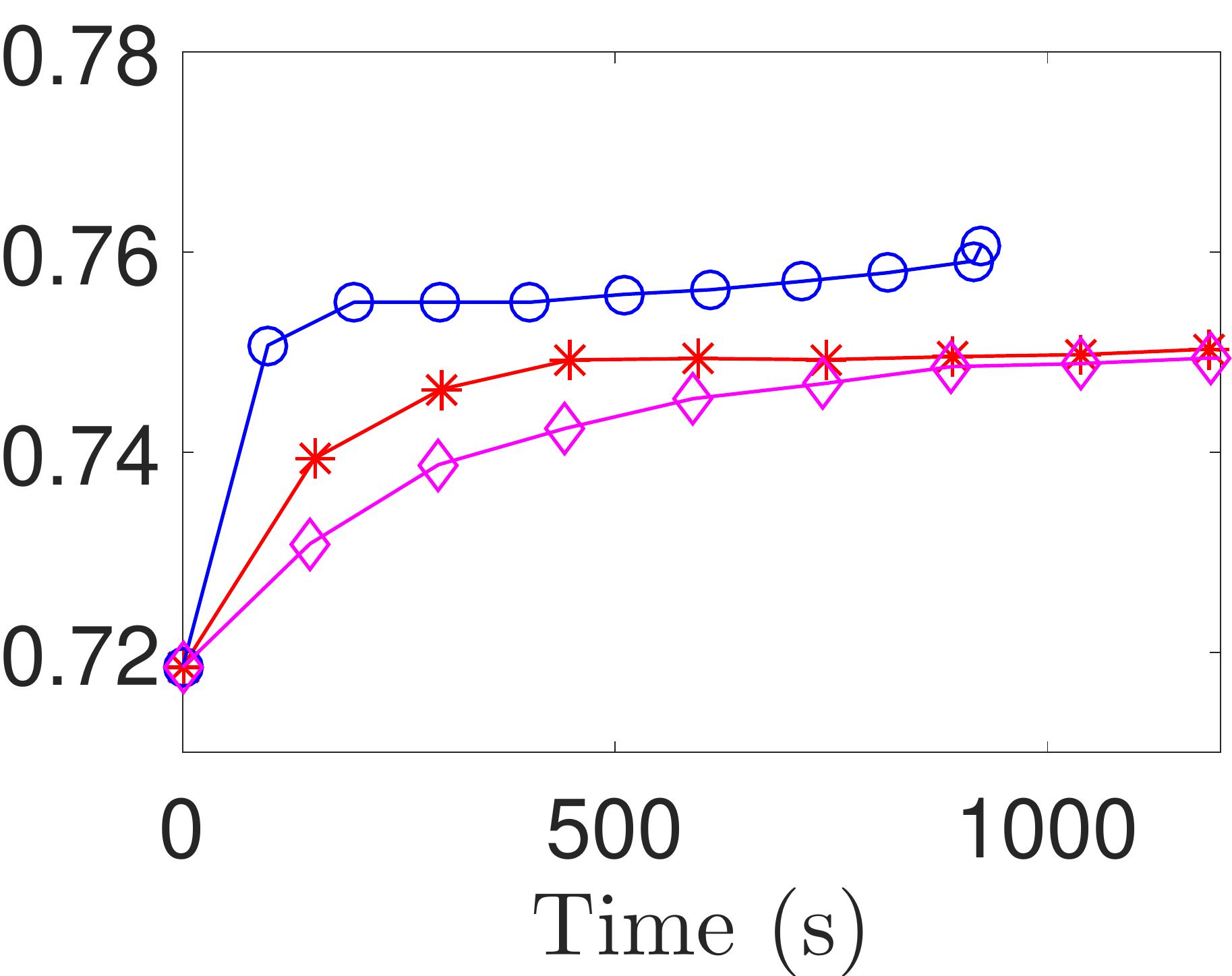}}\\
\hspace*{3cm}(a) $\rho_s=0.1$ \hspace*{5cm} (b) $\rho_s = 0.3$
\caption{ Plots of EV produced by VR, SMM and SGD versus {number of data passes} and {time}.\label{fig:subspace_rec}}
\end{figure}

\newcommand{\BIBdecl}{\setlength{\itemsep}{-.0cm}}
\bibliographystyle{IEEEtran}
{{\footnotesize\bibliography{dataset,mach_learn,math_opt,ORNMF_ref,RNMF_ref,stat_ref,stoc_ref}}

\clearpage
\center{\Large\bf Supplemental Material for ``A Unified Framework for\\[.1cm] Stochastic Matrix Factorization via Variance Reduction''}

\renewcommand{\thedefinition}{S-\arabic{definition}}
\renewcommand{\thelemma}{S-\arabic{lemma}}
\renewcommand{\thecorollary}{S-\arabic{corollary}}
\renewcommand{\theequation}{S-\arabic{equation}}
\renewcommand{\thesection}{S-\arabic{section}}
\renewcommand{\theremark}{S-\arabic{remark}}
\renewcommand{\thefigure}{S-\arabic{figure}}
\renewcommand{\thetable}{S-\arabic{table}}
\renewcommand{\thealgorithm}{S-\arabic{algorithm}}

\setcounter{lemma}{0}
\setcounter{equation}{0}
\setcounter{theorem}{0}
\setcounter{section}{0}
\setcounter{remark}{0}

\flushleft
\section{Proof of Theorem~\ref{thm:main} }\label{sec:proof_main_thm}
For notational convenience, define $\psi'\defeq \psi+\delta_\calC$, then 
\begin{equation}
\bW^{s,t+1}=\prox_{\eta\psi'}(\bW^{s,t}-\eta\bV_{s,t}). \label{eq:stoc_proxStep}
\end{equation}
 In addition, define 
\begin{equation}
\tilbW^{s,t+1}\defeq\prox_{\eta\psi'}(\bW^{s,t}-\eta\nabla g(\bW^{s,t})),\label{eq:batch_proxStep}
\end{equation} 
then 
\begin{equation}
\Gamma_{f',\eta}(\bW^{s,t}) = \frac{1}{\eta}(\bW^{s,t}-\tilbW^{s,t+1}). \label{eq:newDef_Gamma}
\end{equation}
Applying Lemma~\ref{lem:Lips_bound} to both \eqref{eq:stoc_proxStep} and \eqref{eq:batch_proxStep}, we have that for any $\bZ_1,\bZ_2\in\bbR^{d\times k}$,
\begin{align}
f'(\bW^{s,t+1}) &\le f'(\bZ_1) + \lrangle{\bW^{s,t+1}-\bZ_1}{\nabla g(\bW^{s,t})-\bV_{s,t}} + \left(\frac{L}{2}-\frac{1}{2\eta}\right)\normt{\bW^{s,t+1}-\bW^{s,t}}^2\nn\\
&\hspace{4cm}+\left(\frac{L}{2}+\frac{1}{2\eta}\right)\normt{\bZ_1-\bW^{s,t}}^2 -\frac{1}{2\eta} \normt{\bW^{s,t+1}-\bZ_1}^2,\label{eq:ub_stoc}\\
f'(\tilbW^{s,t+1}) &\le f'(\bZ_2) + \left(\frac{L}{2}-\frac{1}{2\eta}\right)\normt{\tilbW^{s,t+1}-\bW^{s,t}}^2 + \left(\frac{L}{2}+\frac{1}{2\eta}\right)\normt{\bZ_2-\bW^{s,t}}^2 \nn\\
&\hspace{8.5cm}- \frac{1}{2\eta} \normt{\tilbW^{s,t+1}-\bZ_2}^2. \label{eq:ub_batch}
\end{align}
By first setting $\bZ_1=\tilbW^{s,t+1}$ in~\eqref{eq:ub_stoc} and $\bZ_2=\bW^{s,t}$ in \eqref{eq:ub_batch} and then summing both inequalities, we have 
\begin{align}
f'(\bW^{s,t+1}) &\le f'(\bW^{s,t}) + \lrangle{\bW^{s,t+1}-\tilbW^{s,t+1}}{\nabla g(\bW^{s,t})-\bV_{s,t}} + \left(\frac{L}{2}-\frac{1}{2\eta}\right)\normt{\bW^{s,t+1}-\bW^{s,t}}^2\nn\\
&\hspace{4cm} +\left(L-\frac{1}{2\eta}\right)\normt{\tilbW^{s,t+1}-\bW^{s,t}}^2 -\frac{1}{2\eta} \normt{\bW^{s,t+1}-\tilbW^{s,t+1}}^2. \label{eq:succ_inner}
\end{align}
Using Lemma~\ref{lem:inner_prod_bound}, we have 
\begin{align}
\lrangle{\bW^{s,t+1}-\tilbW^{s,t+1}}{\nabla g(\bW^{s,t})-\bV_{s,t}} &\le \frac{1}{2\eta} \normt{\bW^{s,t+1}-\tilbW^{s,t+1}}^2 + \frac{\eta}{2} \normt{{\nabla g(\bW^{s,t})-\bV_{s,t}}}^2. 
\end{align}
Therefore, by Lemma~\ref{lem:var_bound}, we have
\begin{align}
\bbE_{\calB_{s,t}}\left[\lrangle{\bW^{s,t+1}-\tilbW^{s,t+1}}{\nabla g(\bW^{s,t})-\bV_{s,t}}\vert\calF_{s,t}\right]&\le \frac{1}{2\eta} \bbE_{\calB_{s,t}}[\normt{\bW^{s,t+1}-\tilbW^{s,t+1}}^2\vert\calF_{s,t}]\nn\\
&\hspace{2.5cm}+ \frac{\eta L^2}{2} \alpha(n,b) \normt{\bW^{s,t}-\bW^{s,0}}^2.\label{eq:exp_inner_prod}
\end{align}
Taking expectations on both sides of~\eqref{eq:succ_inner} and  making use of \eqref{eq:exp_inner_prod}, we have
\begin{align}
\bbE_{\calB_{s,t}}[f'(\bW^{s,t+1})\vert\calF_{s,t}]&\le f'(\bW^{s,t}) + \frac{\eta L^2}{2} \alpha(n,b) \normt{\bW^{s,t}-\bW^{s,0}}^2\nn\\
 &+ \left(\frac{L}{2}-\frac{1}{2\eta}\right)\bbE_{\calB_{s,t}}[\normt{\bW^{s,t+1}-\bW^{s,t}}^2\vert\calF_{s,t}]+\left(L-\frac{1}{2\eta}\right)\normt{\tilbW^{s,t+1}-\bW^{s,t}}^2.  \label{eq:bound_exp_f}
\end{align}
Now, define a surrogate function $\hatf_{s,t}(\bW)\defeq f'(\bW)+\zeta_t\normt{\bW-\bW^{s,0}}^2$, where the sequence $\{\zeta_t\}_{t=0}^m$ is given by the recursion
\begin{equation}
\zeta_t = (1+1/m)\zeta_{t+1} +\frac{\eta L^2}{2}\alpha(n,b), \quad \zeta_m = 0. \label{eq:recursion}
\end{equation}
In particular, we have \begin{equation}
\hatf_{s,0}(\bW^{s,0}) = f'(\bW^{s,0}) \quad \mbox{and}\quad \hatf_{s,m}(\bW^{s,m}) = f'(\bW^{s,m}).\label{eq:def_t0_m}
\end{equation} 
From the recursion \eqref{eq:recursion}, we also observe that $\{\zeta_t\}_{t=0}^m$ is deceasing and 
\begin{equation}
\zeta_0 = \frac{\left((1+1/m)^m-1\right)m}{2}\alpha(n,b)\eta L^2 \le \frac{\left(e-1\right)m}{2}\alpha(n,b)\eta L^2 \le m\alpha(n,b)\eta L^2. \label{eq:bound_zeta0}
\end{equation}
Again by using Lemma~\ref{lem:inner_prod_bound}, we bound 
\begin{equation}
\hatf_{s,t+1}(\bW^{s,t+1}) \le f'(\bW^{s,t+1})+(1+m)\zeta_{t+1}\normt{\bW^{s,t+1}-\bW^{s,t}}^2+(1+1/m)\zeta_{t+1}\normt{\bW^{s,0}-\bW^{s,t}}^2. \label{eq:bound_hatf}
\end{equation}
Combining \eqref{eq:bound_exp_f} and \eqref{eq:bound_hatf}, we have
\begin{align}
\bbE_{\calB_{s,t}}[\hatf_{s,t+1}(\bW^{s,t+1})\vert\calF_{s,t}]&\le f'(\bW^{s,t}) + \left((1+1/m)\zeta_{t+1}+\frac{\eta L^2}{2}\alpha(n,b)\right)  \normt{\bW^{s,t}-\bW^{s,0}}^2\nn\\
 &\hspace{-3cm}+ \left(\frac{L}{2}-\frac{1}{2\eta}+(1+m)\zeta_{t+1}\right)\bbE_{\calB_{s,t}}[\normt{\bW^{s,t+1}-\bW^{s,t}}^2\vert\calF_{s,t}]+\left(L-\frac{1}{2\eta}\right)\normt{\tilbW^{s,t+1}-\bW^{s,t}}^2\\
 &= \hatf_{s,t}(\bW^{s,t}) + \left(\frac{L}{2}-\frac{1}{2\eta}+(1+m)\zeta_{t+1}\right)\bbE_{\calB_{s,t}}[\normt{\bW^{s,t+1}-\bW^{s,t}}^2\vert\calF_{s,t}]\nn\\
 &\hspace{6cm}+\left(L-\frac{1}{2\eta}\right)\normt{\tilbW^{s,t+1}-\bW^{s,t}}^2. 
\end{align}
Using the conditions $\eta = 1/(\theta L)$ and $\theta(\theta-1)\ge 2m(m+1)\alpha(n,b)$ and \eqref{eq:bound_zeta0}, we have for any $t\in[m]$,
\begin{align}
\frac{L}{2}-\frac{1}{2\eta}+(1+m)\zeta_{t} &\le \frac{1}{2}(1-\theta)L+(1+m)\zeta_0\\
&\le \frac{1}{2}(1-\theta)L+m(m+1)\alpha(n,b)(L/\theta)\\
&\le \frac{L}{2\theta }\left(\theta(1-\theta)+2m(m+1)\alpha(n,b)\right)\\
&\le 0. 
\end{align}
Therefore, by~\eqref{eq:newDef_Gamma}, we have
\begin{align}
\bbE_{\calB_{s,t}}[\hatf_{s,t+1}(\bW^{s,t+1})\vert\calF_{s,t}] \le \hatf_{s,t}(\bW^{s,t}) +\eta\left(\eta L-\frac{1}{2}\right)\normt{\Gamma_{f',\eta}(\bW^{s,t})}^2. \label{eq:toTel_inner}
\end{align}
Telescoping~\eqref{eq:toTel_inner} over $t=0,\ldots,m-1$ and noting~\eqref{eq:def_t0_m}, we have
\begin{align}
\bbE_{\calB_{s,(m]}}[f'(\bW^{s+1,0})\vert\calF_{s,0}] = \bbE_{\calB_{s,(m]}}[f'(\bW^{s,m})\vert\calF_{s,0}] \le f'(\bW^{s,0}) + \eta\left(\eta L-\frac{1}{2}\right)\sum_{t=0}^{m-1}\bbE_{\calB_{s,(t]}}[\normt{\Gamma_{f',\eta}(\bW^{s,t})}^2\vert\calF_{s,0}]. \label{eq:toTel_outer}
\end{align}
For any $S\ge 0$, telescope~\eqref{eq:toTel_outer} over $s=0,1,\ldots,S-1$ and we have
\begin{align}
f^*\le \bbE[f'(\bW^{S,0})] &\le f'(\bW^{0,0}) + \eta\left(\eta L-\frac{1}{2}\right)\sum_{s=0}^{S-1}\sum_{t=0}^{m-1}\bbE[\normt{\Gamma_{f',\eta}(\bW^{s,t})}^2].
\end{align}
If we choose the final basis matrix $\bW_{\rm final}$ uniformly randomly from $\{\bW^{s,t}\}_{s\in (S-1],t\in(m]}$, then 
\begin{equation}
\bbE[\normt{\Gamma_{f',\eta}(\bW_{\rm final})}^2] \le \left(\frac{f'(\bW^{0,0})-f^*}{\eta\left({1}/{2}-\eta L\right)}\right)\frac{1}{mS}\eqa \left(\frac{f(\bW^{0,0})-f^*}{\eta\left({1}/{2}-\eta L\right)}\right)\frac{1}{mS},
\end{equation}
where (a) holds because $\bW^{0,0}\in\calC$. 

\section{Proof of Theorem~\ref{thm:inexact} }\label{sec:proof_inexact}
Define $\hatbV_{s,t}\defeq \bE_{s,t}+\bV_{s,t}$, then 
\begin{equation}
\bW^{s,t+1}=\prox_{\eta\psi'}(\bW^{s,t}-\eta\hatbV_{s,t}). \label{eq:stoc_proxStep_hat}
\end{equation}
Applying Lemma~\ref{lem:Lips_bound} to \eqref{eq:stoc_proxStep_hat}, we have that for any $\bZ_3\in\bbR^{d\times k}$,
\begin{align}
f'(\bW^{s,t+1}) &\le f'(\bZ_3) + \lrangle{\bW^{s,t+1}-\bZ_3}{\nabla g(\bW^{s,t})-\hatbV_{s,t}} + \left(\frac{L}{2}-\frac{1}{2\eta}\right)\normt{\bW^{s,t+1}-\bW^{s,t}}^2\nn\\
&\hspace{4cm}+\left(\frac{L}{2}+\frac{1}{2\eta}\right)\normt{\bZ_3-\bW^{s,t}}^2 -\frac{1}{2\eta} \normt{\bW^{s,t+1}-\bZ_3}^2.\label{eq:ub_stoc_hat}
\end{align}
By setting $\bZ_3=\tilbW^{s,t+1}$ in~\eqref{eq:ub_stoc_hat} and then summing \eqref{eq:ub_stoc_hat} and \eqref{eq:ub_batch}, we have 
\begin{align}
f'(\bW^{s,t+1}) &\le f'(\bW^{s,t}) + \lrangle{\bW^{s,t+1}-\tilbW^{s,t+1}}{\nabla g(\bW^{s,t})-\hatbV_{s,t}} + \left(\frac{L}{2}-\frac{1}{2\eta}\right)\normt{\bW^{s,t+1}-\bW^{s,t}}^2\nn\\
&\hspace{4cm} +\left(L-\frac{1}{2\eta}\right)\normt{\tilbW^{s,t+1}-\bW^{s,t}}^2 -\frac{1}{2\eta} \normt{\bW^{s,t+1}-\tilbW^{s,t+1}}^2. \label{eq:succ_inner_hat}
\end{align}
By repeatedly applying Lemma~\ref{lem:inner_prod_bound}, we have 
\begin{align*}
\lrangle{\bW^{s,t+1}-\tilbW^{s,t+1}}{\nabla g(\bW^{s,t})-\hatbV_{s,t}} &\le \frac{1}{2\eta} \normt{\bW^{s,t+1}-\tilbW^{s,t+1}}^2 + {\eta} \normt{{\nabla g(\bW^{s,t})-\bV_{s,t}}}^2 + \eta\normt{\bE_{s,t}}^2. 
\end{align*}
Therefore, by Lemma~\ref{lem:var_bound}, we have
\begin{align}
\bbE_{\calB_{s,t}}[\langle\bW^{s,t+1}-\tilbW^{s,t+1},\nabla g(\bW^{s,t})-\hatbV_{s,t}\rangle\vert\calF_{s,t}]&\le \frac{1}{2\eta} \bbE_{\calB_{s,t}}[\normt{\bW^{s,t+1}-\tilbW^{s,t+1}}^2\vert\calF_{s,t}]+ \eta\normt{\bE_{s,t}}^2\nn\\
&\hspace{3cm}+ {\eta L^2} \alpha(n,b) \normt{\bW^{s,t}-\bW^{s,0}}^2.\label{eq:exp_inner_prod_hat}
\end{align}
Taking expectations on both sides of~\eqref{eq:succ_inner_hat} and  making use of \eqref{eq:exp_inner_prod_hat}, we have
\begin{align}
\bbE_{\calB_{s,t}}[f'(\bW^{s,t+1})\vert\calF_{s,t}]&\le f'(\bW^{s,t}) + {\eta L^2} \alpha(n,b) \normt{\bW^{s,t}-\bW^{s,0}}^2+\eta\normt{\bE_{s,t}}^2\nn\\
 &+ \left(\frac{L}{2}-\frac{1}{2\eta}\right)\bbE_{\calB_{s,t}}[\normt{\bW^{s,t+1}-\bW^{s,t}}^2\vert\calF_{s,t}]+\left(L-\frac{1}{2\eta}\right)\normt{\tilbW^{s,t+1}-\bW^{s,t}}^2.  \label{eq:bound_exp_f_hat}
\end{align}
Now, define a surrogate function $\hatf_{s,t}(\bW)\defeq f'(\bW)+\zeta_t\normt{\bW-\bW^{s,0}}^2$, where the sequence $\{\zeta_t\}_{t=0}^m$ is given by the recursion
\begin{equation}
\zeta_t = (1+1/m)\zeta_{t+1} +{\eta L^2}\alpha(n,b), \quad \zeta_m = 0. \label{eq:recursion_hat}
\end{equation}
In particular, we have \begin{equation}
\hatf_{s,0}(\bW^{s,0}) = f'(\bW^{s,0}) \quad \mbox{and}\quad \hatf_{s,m}(\bW^{s,m}) = f'(\bW^{s,m}).\label{eq:def_t0_m_hat}
\end{equation} 
From the recursion \eqref{eq:recursion_hat}, we also observe that $\{\zeta_t\}_{t=0}^m$ is deceasing and 
\begin{equation}
\zeta_0 = {\left((1+1/m)^m-1\right)m}\alpha(n,b)\eta L^2 \le {\left(e-1\right)m}\alpha(n,b)\eta L^2 \le 2m\alpha(n,b)\eta L^2. \label{eq:bound_zeta0_hat}
\end{equation}
Again by using Lemma~\ref{lem:inner_prod_bound}, we bound 
\begin{equation}
\hatf_{s,t+1}(\bW^{s,t+1}) \le f'(\bW^{s,t+1})+(1+m)\zeta_{t+1}\normt{\bW^{s,t+1}-\bW^{s,t}}^2+(1+1/m)\zeta_{t+1}\normt{\bW^{s,0}-\bW^{s,t}}^2. \label{eq:bound_hatf_hat}
\end{equation}
Combining \eqref{eq:bound_exp_f_hat} and \eqref{eq:bound_hatf_hat}, we have
\begin{align}
\bbE_{\calB_{s,t}}[\hatf_{s,t+1}(\bW^{s,t+1})\vert\calF_{s,t}]&\le f'(\bW^{s,t}) + \left((1+1/m)\zeta_{t+1}+{\eta L^2}\alpha(n,b)\right)  \normt{\bW^{s,t}-\bW^{s,0}}^2+ \eta\normt{\bE_{s,t}}^2\nn\\
 &\hspace{-3cm}+ \left(\frac{L}{2}-\frac{1}{2\eta}+(1+m)\zeta_{t+1}\right)\bbE_{\calB_{s,t}}[\normt{\bW^{s,t+1}-\bW^{s,t}}^2\vert\calF_{s,t}]+\left(L-\frac{1}{2\eta}\right)\normt{\tilbW^{s,t+1}-\bW^{s,t}}^2\\
 &= \hatf_{s,t}(\bW^{s,t}) + \left(\frac{L}{2}-\frac{1}{2\eta}+(1+m)\zeta_{t+1}\right)\bbE_{\calB_{s,t}}[\normt{\bW^{s,t+1}-\bW^{s,t}}^2\vert\calF_{s,t}]\nn\\
 &\hspace{4cm}+\left(L-\frac{1}{2\eta}\right)\normt{\tilbW^{s,t+1}-\bW^{s,t}}^2+ \eta\normt{\bE_{s,t}}^2. 
\end{align}
Using the conditions $\eta = 1/(\theta L)$ and $\theta(\theta-1)\ge 4m(m+1)\alpha(n,b)$ and \eqref{eq:bound_zeta0}, we have for any $t\in[m]$,
\begin{align}
\frac{L}{2}-\frac{1}{2\eta}+(1+m)\zeta_{t} &\le \frac{1}{2}(1-\theta)L+(1+m)\zeta_0\\
&\le \frac{1}{2}(1-\theta)L+2m(m+1)\alpha(n,b)(L/\theta)\\
&\le \frac{L}{2\theta }\left(\theta(1-\theta)+4m(m+1)\alpha(n,b)\right)\\
&\le 0. 
\end{align}
Therefore, by~\eqref{eq:newDef_Gamma}, we have
\begin{align}
\bbE_{\calB_{s,t}}[\hatf_{s,t+1}(\bW^{s,t+1})\vert\calF_{s,t}] \le \hatf_{s,t}(\bW^{s,t}) +\eta\left(\eta L-\frac{1}{2}\right)\normt{\Gamma_{f',\eta}(\bW^{s,t})}^2+ \eta\normt{\bE_{s,t}}^2. \label{eq:toTel_inner_hat}
\end{align}
Telescoping~\eqref{eq:toTel_inner} over $t=0,\ldots,m-1$ and noting~\eqref{eq:def_t0_m_hat}, we have
\begin{align}
\bbE_{\calB_{s,(m]}}[f'(\bW^{s+1,0})\vert\calF_{s,0}] &= \bbE_{\calB_{s,(m]}}[f'(\bW^{s,m})\vert\calF_{s,0}] \nn\\
&\le f'(\bW^{s,0}) + \eta\left(\eta L-\frac{1}{2}\right)\sum_{t=0}^{m-1}\bbE_{\calB_{s,(t]}}[\normt{\Gamma_{f',\eta}(\bW^{s,t})}^2\vert\calF_{s,0}]+ \eta\sum_{t=0}^{m-1}\normt{\bE_{s,t}}^2. \label{eq:toTel_outer_hat}
\end{align}
For any $S\ge 0$, telescope~\eqref{eq:toTel_outer_hat} over $s=0,1,\ldots,S-1$ and we have
\begin{align}
f^*\le \bbE[f'(\bW^{S,0})] &\le f'(\bW^{0,0}) + \eta\left(\eta L-\frac{1}{2}\right)\sum_{s=0}^{S-1}\sum_{t=0}^{m-1}\bbE[\normt{\Gamma_{f',\eta}(\bW^{s,t})}^2] + \eta\sum_{s=0}^{S-1}\sum_{t=0}^{m-1}\normt{\bE_{s,t}}^2.
\end{align}
Since $\normt{\bE_{s,t}}=o(1/\!\sqrt{ms+t})$, for any $S\ge 1$, there exists a constant $E<\infty$ (independent of $S$) such that $\sum_{s=0}^{S-1}\sum_{t=0}^{m-1}\normt{\bE_{s,t}}^2 \le E$.
Therefore by using option I to choose $\bW_{\rm final}$, we have
\begin{equation}
\bbE[\normt{\Gamma_{f',\eta}(\bW_{\rm final})}^2] \le \left(\frac{f'(\bW^{0,0})-f^*+\eta E}{\eta\left({1}/{2}-\eta L\right)}\right)\frac{1}{mS}.
\end{equation}

\section{Proof of Lemma~\ref{lem:regularity}} \label{sec:proof_regularity}

Define $(\bh^*(\by,\bW),\br^*(\by,\bW))\defeq \min_{\bh\in\calH,\br\in\calR}\tilell_2(\by,\bW,\bh,\br)$. By assumption~\ref{assump:compact_set}, if $\calH$ is not compact, then $\varphi$ will be coercive. This implies the boundedness of $\bh^*(\by,\bW)$. A similar argument also applies to $\br^*\!(\by,\!\bW)$. Thus it is equivalent to consider minimizing $\tilell_2$ over some compact sets $\calH'\subseteq\calH$ and $\calR'\subseteq\calR$.

It is easy to verify that the following conditions hold:
\begin{enumerate}
\item $\tilell_2(\cdot,\cdot, \bh, \br)$ is differentiable on $\bbR^d\times\bbR^{d\times k}$, for each $(\bh,\br)\in\bbR^k\times\bbR^d$,
\item $\tilell_2(\cdot,\cdot,\cdot,\cdot)$ and $(\by,\bW, \bh, \br)\mapsto\nabla_{(\by,\bW)}\tilell_2(\by,\bW, \bh, \br)$ are continuous on $\bbR^d\times\bbR^{d\times k}\times\bbR^k\times\bbR^d$,
\item For any $\by\in\bbR^d$ and $\bW\!\in\!\bbR^{d\times k}$, the minimizer $(\bh^*(\by,\bW),\br^*(\by,\bW))$ is unique, due to assumption~\ref{assump:sc}. 
\end{enumerate}
Thus, we can invoke Danskin's theorem (see Lemma~\ref{lem:Danskin}) to conclude that $\ell_2(\cdot,\cdot)$ is differentiable on $\bbR^d\times\bbR^{d\times k}$, and compute 
\begin{equation}
\nabla_\bW\ell_2(\by,\bW) = \left(\bW\bh^*(\by,\bW)+\br^*(\by,\bW)-\by\right)\bh^*(\by,\bW)^T. \label{eq:nablaW_ell}
\end{equation}
Furthermore, we can show $\bh^*(\by,\bW)$ and $\br^*(\by,\bW)$ are both continuous on $\calY\times\calC$ by the maximum theorem (see Lemma~\ref{lem:maximum}), since the conditions in this theorem are trivially satisfied in our case. Therefore, from \eqref{eq:nablaW_ell}, we can see that $(\by,\bW)\mapsto\nabla_\bW\ell_2(\by,\bW)$ is continuous on $\calY\times \calC$. 

We next show that for all $\by\in\calY$, both $\bh^*(\by,\cdot)$ and $\br^*(\by,\cdot)$ are Lipschitz on $\calC$, with Lipschitz constants independent of $\by$. 
Fix any $\by_1,\by_2\in\calY$ and $\bW_1,\bW_2\in\calC$. Define 
\begin{align*}
D(\bh,\br) &\defeq \tilell_2(\by_1,\bW_1,\bh,\br) - \tilell_2(\by_2,\bW_2,\bh,\br)\\
& =\frac{1}{2}\norm{\by_1-\bY_1\bb(\bh,\br)}_2^2 - \frac{1}{2}\norm{\by_2-\bY_2\bb(\bh,\br)}_2^2
\end{align*}
where $\bY_i = [\bW_i\;\; \bI]$, $i=1,2$ and $\bb(\bh,\br) = [\bh^T\;\;\br^T]^T$. By Lemma~\ref{lem:Lips_b}, we have for all $(\bh_1,\br_1)$ and $(\bh_2,\br_2)$ in $\calH'\times\calR$, there exist constants $c_1,c_2>0$ (independent of $\by_1$, $\by_2$, $\bY_1$ and $\bY_2$) such that 
\begin{align}
\abs{D(\bh_1,\br_1)-D(\bh_2,\br_2)} &\le \left(c_1\norm{\by_1-\by_2}_2 + c_2\norm{\bY_1-\bY_2}_2\right)\norm{\bb(\bh_1,\br_1)-\bb(\bh_2,\br_2)}_2.
\end{align}
In particular, we have
\begin{equation}
\abs{D(\bh^*_1,\br^*_1)-D(\bh^*_2,\br^*_2)} \le \left(c_1\norm{\by_1-\by_2}_2 + c_2\norm{\bY_1-\bY_2}_2\right)\norm{\bb(\bh^*_1,\br^*_1)-\bb(\bh^*_2,\br^*_2)}_2 \label{eq:Lips_upperbd}
\end{equation}
where $\bh^*_i=\bh^*(\by_i,\bW_i)$ and $\br^*_i=\br^*(\by_i,\bW_i)$, $i=1,2$. On the other hand, by Assumption~\ref{assump:sc}, there exist a constant $\nu>0$ such that
\begin{align*}
\abs{D(\bh^*_2,\br^*_2) - D(\bh^*_1,\br^*_1)} &=\abs{\tilell_2(\by_1,\bW_1,\bh^*_2,\br^*_2) - \tilell_2(\by_2,\bW_2,\bh^*_2,\br^*_2) - \tilell_2(\by_1,\bW_1,\bh^*_1,\br^*_1) + \tilell_2(\by_2,\bW_2,\bh^*_1,\br^*_1)}\\
&= (\tilell_2(\by_1,\bW_1,\bh^*_2,\br^*_2)-\tilell_2(\by_1,\bW_1,\bh^*_1,\br^*_1)) + (\tilell_2(\by_2,\bW_2,\bh^*_1,\br^*_1) - \tilell_2(\by_2,\bW_2,\bh^*_2,\br^*_2))\\
&\ge \nu\norm{\bb(\bh^*_1,\br^*_1)-\bb(\bh^*_2,\br^*_2)}_2^2.\numberthis \label{eq:sc_lowerbd}
\end{align*}
Combining \eqref{eq:Lips_upperbd} and \eqref{eq:sc_lowerbd}, we have
\begin{equation}
\norm{\bb(\bh^*_1,\br^*_1)-\bb(\bh^*_2,\br^*_2)}_2 \le c'_1\norm{\by_1-\by_2}_2 + c'_2\norm{\bW_1-\bW_2}_F
\end{equation}
where $c'_i = c_i/\nu$, $i=1,2$. 
Since
\begin{align}
\max\{\norm{\bh^*_1-\bh^*_2}_2,\norm{\br^*_1-\br^*_2}_2\} &\le \norm{\bb(\bh^*_1,\br^*_1)-\bb(\bh^*_2,\br^*_2)}_2\quad\mbox{and}\\
c'_1\norm{\by_1-\by_2}_2 + c'_2\norm{\bW_1-\bW_2}_F &\le 2\max(c'_1,c'_2)\norm{[\by_1\;\;\bW_1]-[\by_2 \;\;\bW_2]}_F,
\end{align}
we have 
\begin{align}
\max\{\norm{\bh^*_1-\bh^*_2}_2,\norm{\br^*_1-\br^*_2}_2\} \le 2\max(c'_1,c'_2)\norm{[\by_1\;\;\bW_1]-[\by_2 \;\;\bW_2]}_F. 
\end{align}
This indeed shows both $\bh^*(\cdot,\cdot)$ and $\br^*(\cdot,\cdot)$ are Lipschitz on $\calY\times\calC$. 

To show the Lipschitz continuity of $\bW\mapsto\nabla_\bW\ell_2(\by,\bW)$ on $\calC$, consider any $\bW_1,\bW_2 \in\calC$ and any $\by\in\calY$. From \eqref{eq:nablaW_ell}, we have 
\begin{align*}
\norm{\nabla_{\bW}\ell_2(\by,\bW_1) - \nabla_{\bW}\ell_2(\by,\bW_2)}_F &\le \norm{\bW_1\bh^*(\by,\bW_1)\bh^*(\by,\bW_1)^T - \bW_2\bh^*(\by,\bW_2)\bh^*(\by,\bW_2)^T}_F \\
&+ \norm{\br^*(\by,\bW_1)\bh^*(\by,\bW_1)^T - \br^*(\by,\bW_2)\bh^*(\by,\bW_2)^T}_F \\
&+ \norm{\by\bh^*(\by,\bW_1)^T - \by\bh^*(\by,\bW_2)^T}_F. \numberthis \label{eq:norm_nabla_ell}
\end{align*}
Define $c\defeq 2\max(c'_1,c'_2)$. We bound each term on the RHS of \eqref{eq:norm_nabla_ell} as follows
\begin{flalign*}
&\norm{\bW_1\bh^*(\by,\bW_1)\bh^*(\by,\bW_1)^T - \bW_2\bh^*(\by,\bW_2)\bh^*(\by,\bW_2)^T}_F& \\
\le \;&\norm{\bW_1}_F\norm{\bh^*(\by,\bW_1)}_2\norm{\bh^*(\by,\bW_1)-\bh^*(\by,\bW_2)}_2 + \norm{\bW_1\bh^*(\by,\bW_1) - \bW_2\bh^*(\by,\bW_2)}_2\norm{\bh^*(\by,\bW_2)}_2&\\
\le \;&\left(c\norm{\bW_1}_F\norm{\bh^*(\by,\bW_1)}_2 + \norm{\bh^*(\by,\bW_1)}_2\norm{\bh^*(\by,\bW_2)}_2 + c\norm{\bW_2}_F\norm{\bh^*(\by,\bW_2)}_2\right)\norm{\bW_1-\bW_2}_F,
\end{flalign*}
\begin{flalign*}
&\norm{\br^*(\by,\bW_1)\bh^*(\by,\bW_1)^T - \br^*(\by,\bW_2)\bh^*(\by,\bW_2)^T}_F&\\
\le\;&\norm{\br^*(\by,\bW_1)}_2\norm{\bh^*(\by,\bW_1)-\bh^*(\by,\bW_2)}_2 + \norm{\br^*(\by,\bW_1)-\br^*(\by,\bW_2)}_2\norm{\bh^*(\by,\bW_2)}_2&\\
\le\;&c(\norm{\br^*(\by,\bW_1)}_2 + \norm{\bh^*(\by,\bW_2)}_2)\norm{\bW_1-\bW_2}_F,&
\end{flalign*}
\begin{flalign*}
&\quad\norm{\by\bh^*(\by,\bW_1)^T - \by\bh^*(\by,\bW_2)^T}\le c\norm{\by}_2\norm{\bW_1-\bW_2}_F.&
\end{flalign*}
Using compactness of $\calY$, $\calC$, $\calH'$ and $\calR'$, we complete the proof. 


\section{Proof of Lemma~\ref{lem:var_bound} }\label{sec:proof_var_bound}
From Lemma~\ref{lem:regularity}, we know 
\begin{equation}
\bV_{s,t} = \nabla g_{\calB_{s,t}}(\bW^{s,t}) - \nabla g_{\calB_{s,t}}(\bW^{s,0}) + \nabla g(\bW^{s,0}), 
\end{equation}
where
\begin{equation}
\nabla g_{\calB_{s,t}} (\bW) = \frac{1}{\abs{\calB_{s,t}}} \sum_{i\in\calB_{s,t}} \nabla_\bW \ell_2(\by_i,\bW), \;\forall\,\bW\in\bbR^{d\times k}. 
\end{equation}
Thus 
\begin{align}
\bbE_{\calB_{s,t}}[\normt{\bV_{s,t}-\nabla g(\bW^{s,t})}^2] &\eqa \bbE_{\calB_{s,t}}[\normt{(\nabla g_{\calB_{s,t}}(\bW^{s,t})-\nabla g_{\calB_{s,t}}(\bW^{s,0}))-(\nabla g(\bW^{s,t}))-\nabla g(\bW^{s,0}))}^2]\\
&\leb\frac{n-b}{b(n-1)}\left(\frac{1}{n}\sum_{i\in[n]}\normt{\nabla_\bW\ell_2(\by_i,\bW^{s,t})-\nabla_\bW\ell_2(\by_i,\bW^{s,0})}^2\right)\\
&\lec \frac{n-b}{b(n-1)} L^2\normt{\bW^{s,t}-\bW^{s,0}}^2,
\end{align}
where in (a) we use \eqref{eq:deriv_ell2} in Lemma~\ref{lem:regularity}, in (b) we use Lemma~\ref{lem:uni_samp_wo} and in (c) we use the Lipschitz continuity of $\bW\!\mapsto\!\nabla_\bW\ell_2(\by,\bW)$ in Lemma~\ref{lem:regularity}.

\section{Algorithms for SMF via SMM and SGD}

For completeness, we present the algorithms to solve SMF problems via SMM and (proximal) SGD in Algorithms~\ref{algo:SMF_SMM} and~\ref{algo:SMF_SGD} respectively. 

\begin{algorithm}[t]
\caption{Stochastic Matrix Factorization via SMM~\cite{Mairal_10}} \label{algo:SMF_SMM}
\begin{algorithmic}[1]
\State {\bf Input}: Number of iterations $M$, mini-batch size $b$ 
\State {\bf Initialize} dictionary $\bW^{0}\in\calC$, sufficient statistics $\bA_0:=\vecz_{k\times k}$, $\bB_0:=\vecz_{d\times k}$ 
\State {\bf For} $t = 0, 1, \ldots, M-1$ 
\State \quad Uniformly randomly sample $\calB_{t}\subseteq[n]$ of size $b$ without replacement 
\State \quad Solve $(\bh^{t}_j,\br^{t}_j)\defeq \argmin_{\bh\in\calH,\br\in\calR} \tilell_2(\by_{j},\bW^{t},\bh,\br)$, for any $j\!\in\!\calB_{t}$. 
\State \quad $\bA_{t+1}:= \bA_t + \sum_{j\in\calB_{t}}\bh^t_j(\bh^t_j)^T$, $\bB_{t+1}:=\bB_t + \sum_{j\in\calB_{s,t}}(\by^t_j-\br^t_j)(\bh^t_j)^T$
\State \quad $\bW^{t+1}:=\argmin_{\bW\in\calC}\frac{1}{t}\left(\frac{1}{2}\tr(\bW^T\bW\bA_{t+1})-\tr(\bW^T\bB_{t+1})\right)$
\State {\bf End for}
\end{algorithmic}
\end{algorithm}

\begin{algorithm}[t]
\caption{Stochastic Matrix Factorization via SGD~\cite{Mairal_10,Zhao_17b}} \label{algo:SMF_SGD}
\begin{algorithmic}[1]
\State {\bf Input}: Number of iterations $M$, mini-batch size $b$, step sizes $\{\gamma_t\}_{t\ge 0}$ 
\State {\bf Initialize} dictionary $\bW^{0}\in\calC$ 
\State {\bf For} $t = 0, 1, \ldots, M-1$ 
\State \quad Uniformly randomly sample $\calB_{t}\subseteq[n]$ of size $b$ without replacement 
\State \quad Solve $(\bh^{t}_j,\br^{t}_j)\defeq \argmin_{\bh\in\calH,\br\in\calR} \tilell_2(\by_{j},\bW^{t},\bh,\br)$, for any $j\!\in\!\calB_{s,t}$. 
\State \quad $\bV_{t}:=\frac{1}{b}\sum_{j\in\calB_{t}}(\bW^{t}\bh_j^{t}+\br_j^{t}-\by_j)(\bh_j^{t})^T$
\State \quad $\bW^{t+1}:=\prox_{\eta\psi+\delta_\calC}(\bW^{t}-\gamma_t\bV_{t})$ \label{line:prox_W}
\State {\bf End for}
\end{algorithmic}
\end{algorithm}

\section{Additional Experimental Results and Discussions}\label{sec:exp_discussion}

\begin{figure}[t]\centering
\subfloat[VR-ODL ({\tt CBCL})]{\includegraphics[width=.24\columnwidth,height=.18\columnwidth]{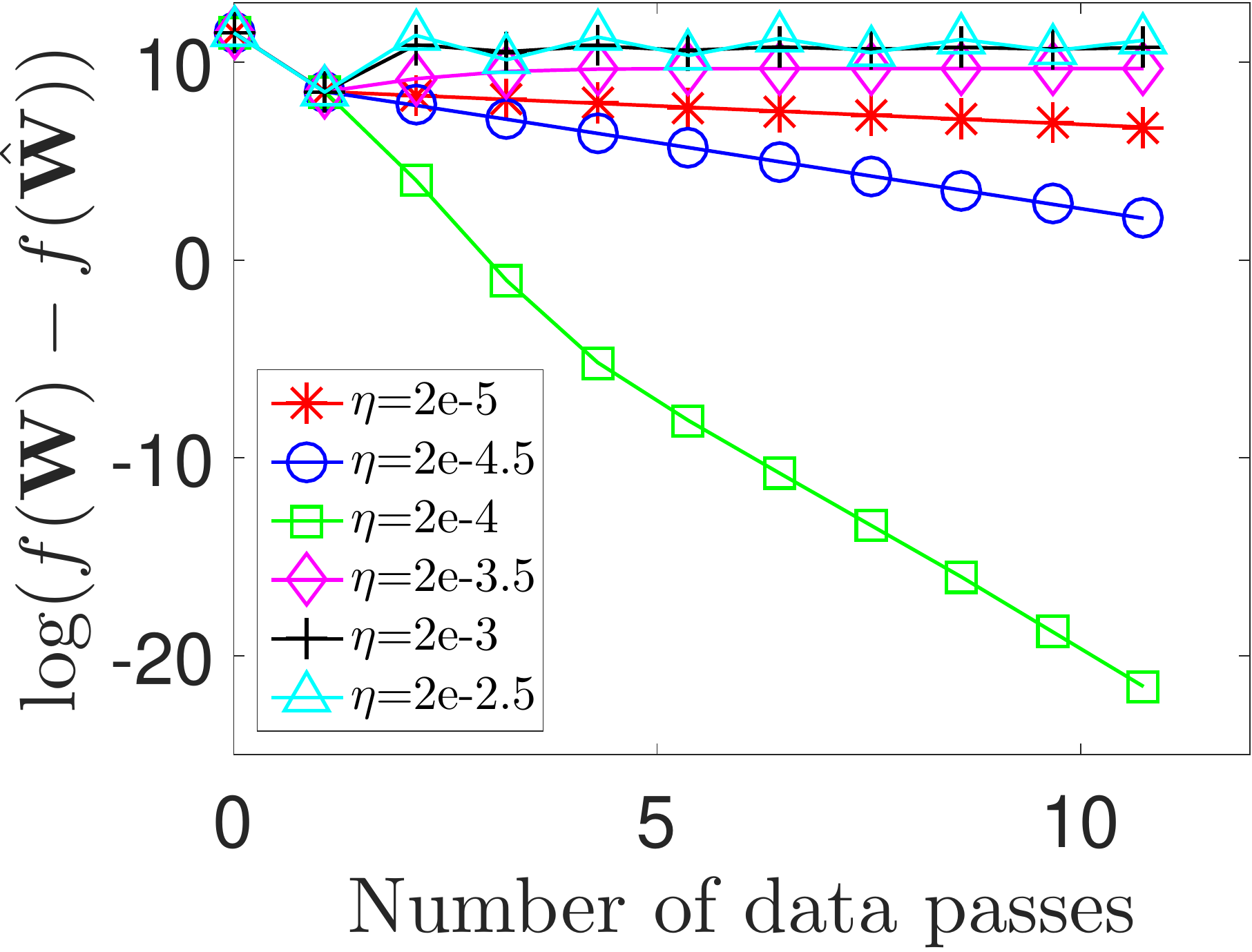}}\hfill
\subfloat[VR-ODL ({\tt MNIST})]{\includegraphics[width=.24\columnwidth,height=.18\columnwidth]{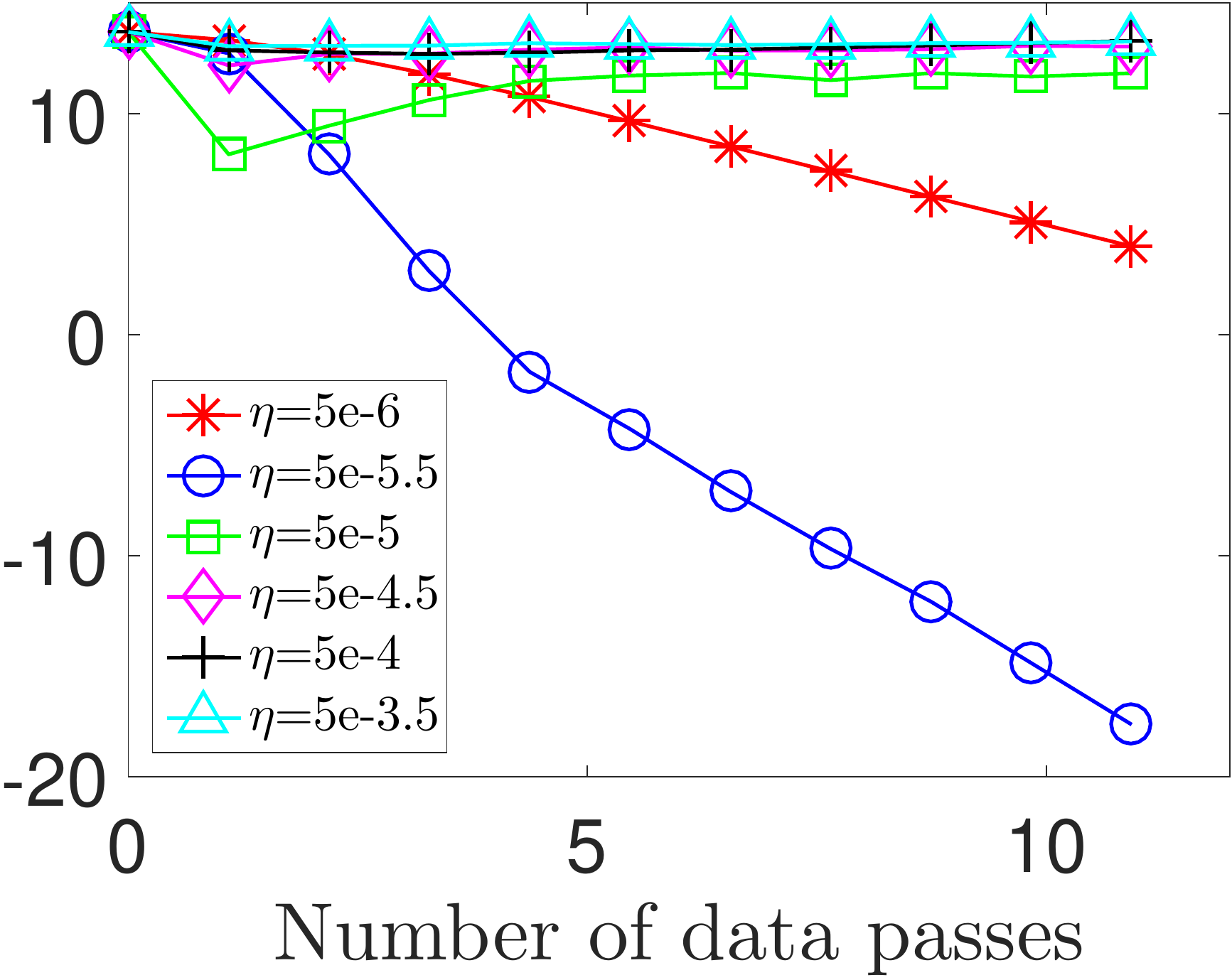}}\hfill
\subfloat[VR-ONMF ({\tt CBCL})]{\includegraphics[width=.24\columnwidth,height=.18\columnwidth]{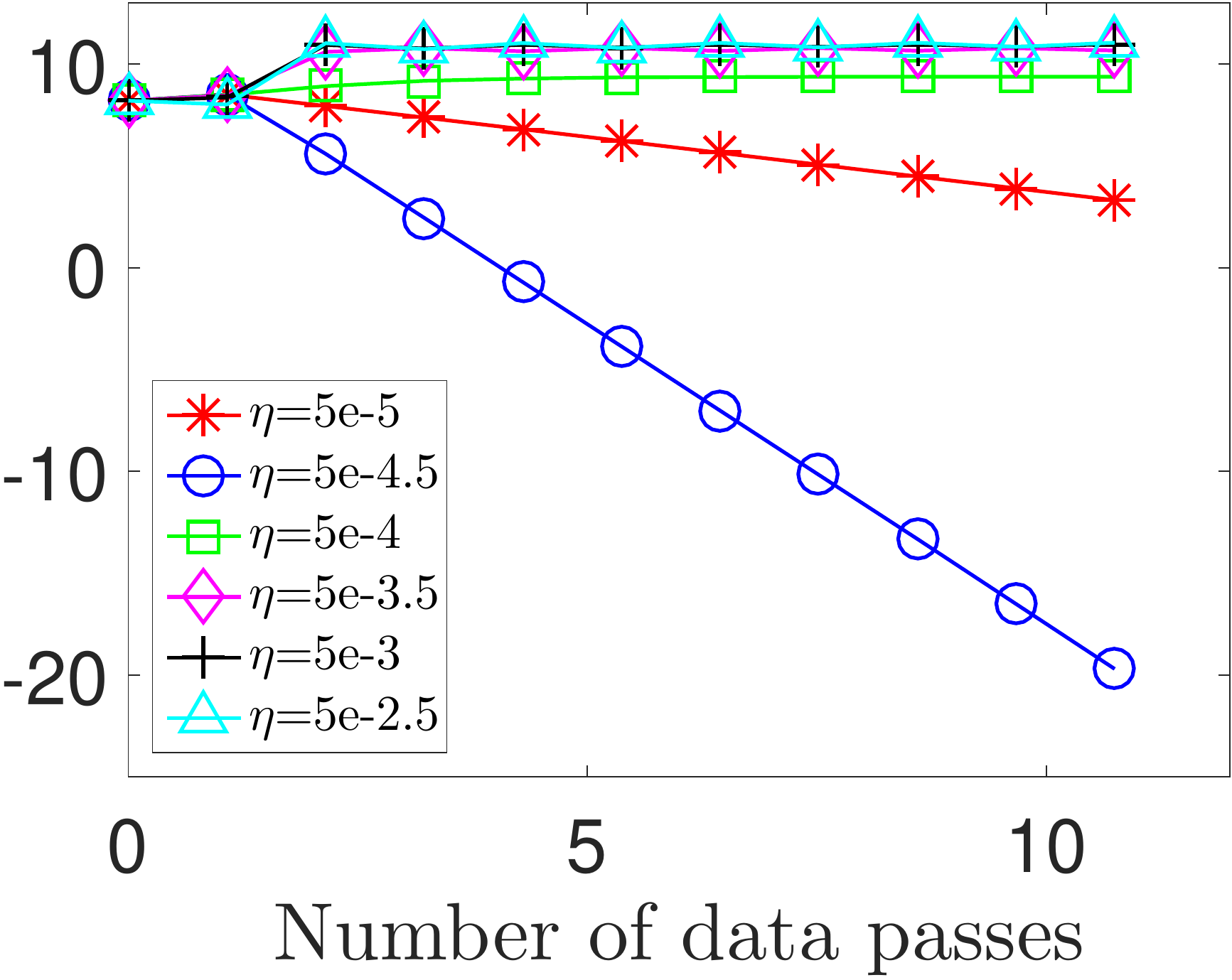}}\hfill
\subfloat[VR-ONMF ({\tt MNIST})]{\includegraphics[width=.24\columnwidth,height=.18\columnwidth]{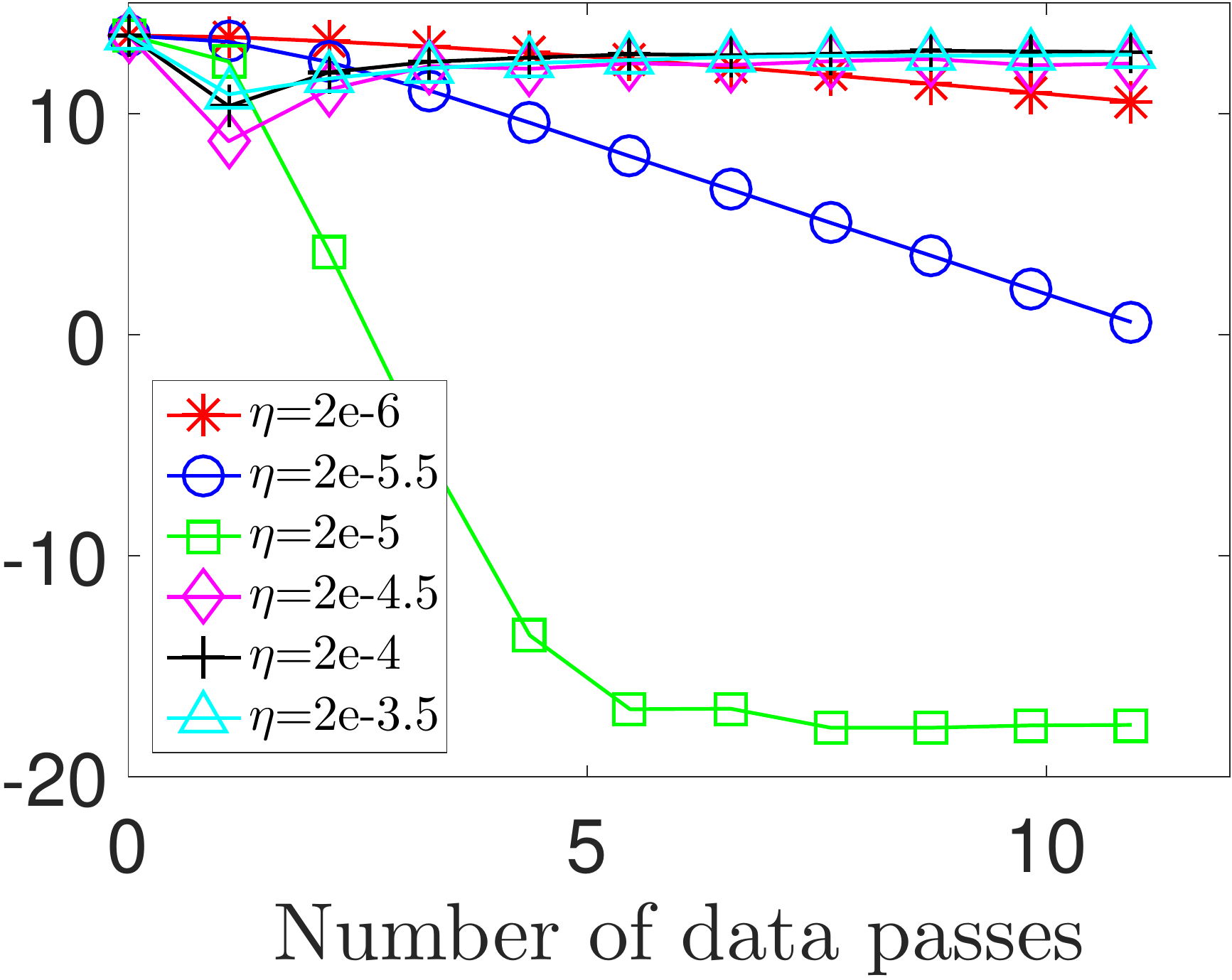}}\\
\subfloat[VR-ORPCA ({\tt Synth})]{\includegraphics[width=.24\columnwidth,height=.175\columnwidth]{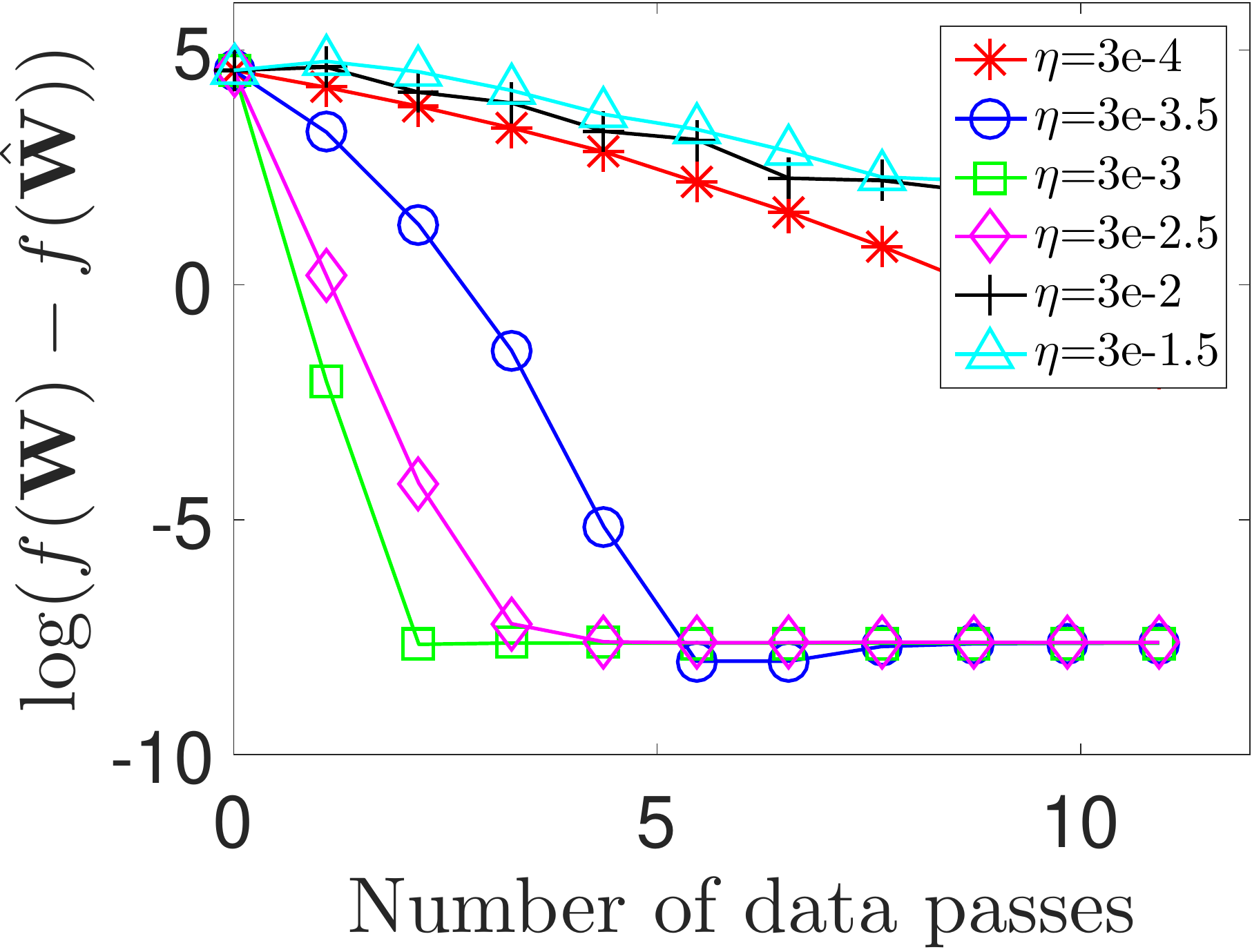}}\hfill
\subfloat[VR-ORPCA ({\tt YaleB})]{\includegraphics[width=.24\columnwidth,height=.175\columnwidth]{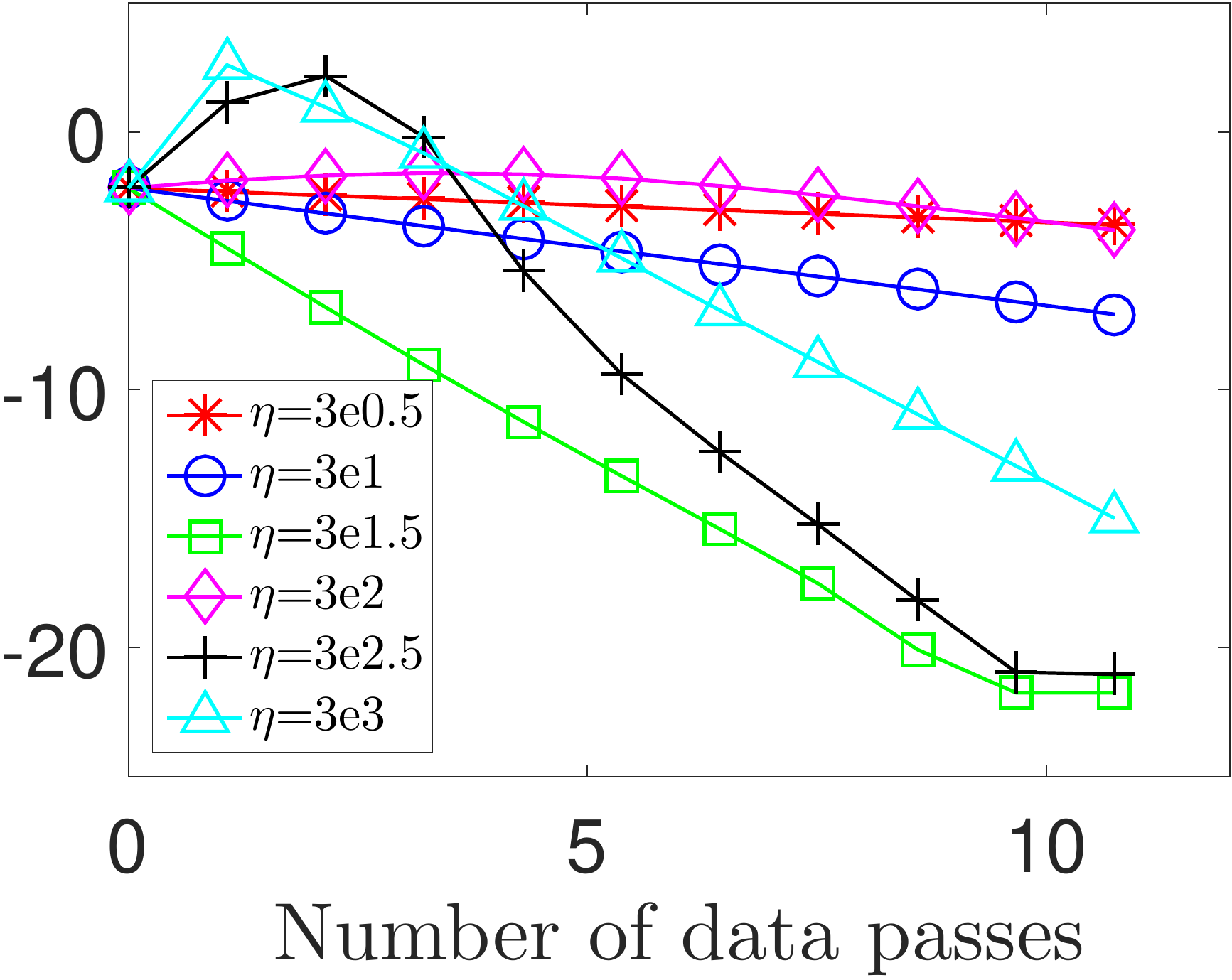}}\hfill
\subfloat[VR-ORNMF ({\tt Synth})]{\includegraphics[width=.24\columnwidth,height=.18\columnwidth]{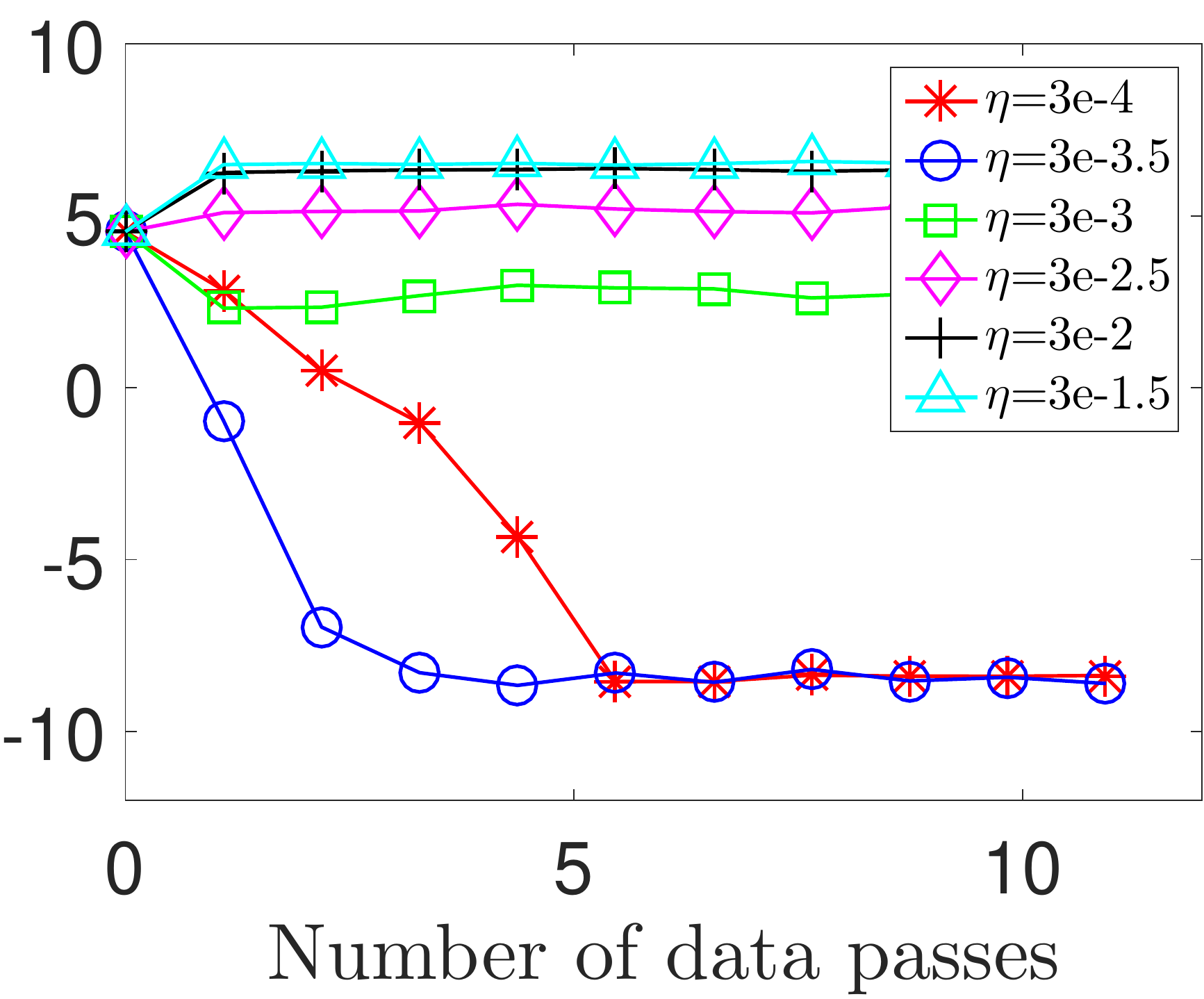}}\hfill
\subfloat[VR-ORNMF ({\tt YaleB})]{\includegraphics[width=.24\columnwidth,height=.18\columnwidth]{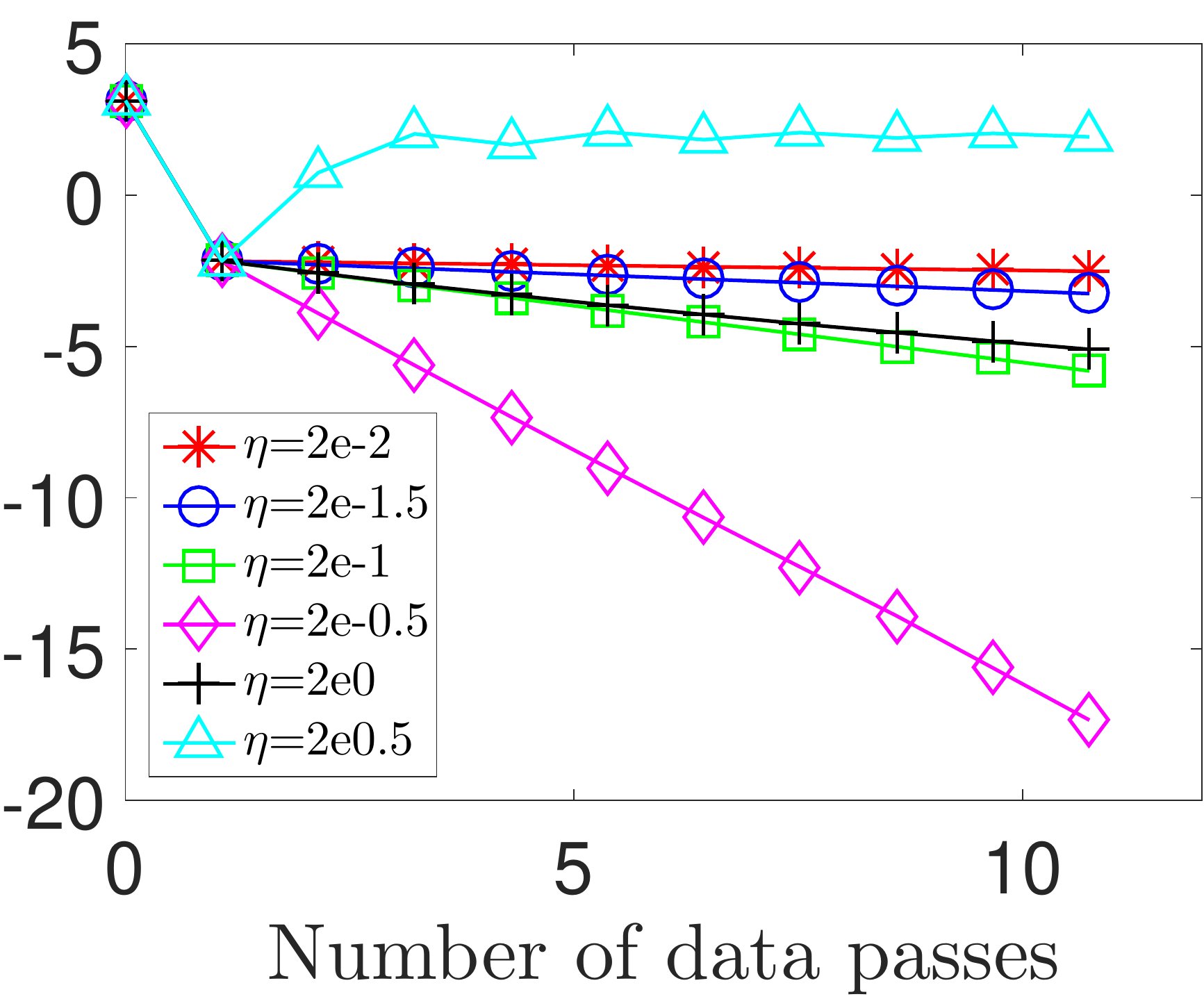}}
\caption{Log suboptimality versus number of passes through data of all the four variance-reduced algorithms (VR-ODL, VR-ONMF, VR-ORPCA and VR-ORNMF) on different datasets ({\tt CBCL}, {\tt MNIST}, {\tt Synth} and {\tt ExtYaleB}) datasets with different step sizes $\eta$. \label{fig:etas}}
\end{figure}

Additional experimental results are shown in Figure~\ref{fig:etas}. Next, we provide some intuitions about the observations in Section 5.3. As shown in~\cite{Raza_16}, under mild conditions, the SMM method can be regarded as a spacial case of the SGD method (with diminishing step sizes). Since the step size vanishes asymptotically, both SMM and SGD will  make minute progress in learning the dictionary $\bW$ after a large number of iterations (or equivalently, data samples). Therefore, they fail to find highly accurate stationary points. In contrast, our variance-reduced method employs a {\em constant} step size, therefore it continues to make non-negligible progress asymptotically, so as to reach a much higher accuracy. Similar arguments can also explain the results in Section~5.4. 


\section{Technical Lemmas}

\begin{lemma}[Danskin's Theorem {\cite{Bon_98}}] \label{lem:Danskin}
Let $\calX$ be a metric space and $\calU$ be a normed vector space. Let $f:\calX\times \calU\to\bbR$ have the following properties
\begin{enumerate}
\item $f(x,\cdot)$ is differentiable on $\calU$, for any $x\in\calX$,
\item $f(x,u)$ and $\nabla_u f(x,u)$ are continuous on $\calX\times\calU$.
\end{enumerate}
Let $\Phi$ be a compact set in $\calX$. Define $v(u) \defeq \inf_{x\in\Phi}f(x,u)$ and $\calS(u) \defeq \argmin_{x\in\Phi}f(x,u)$. 
If for some $u_0\in\calU$, $\calS(u_0) = \{x_0\}$, then $v$ is (Hadamard) differentiable at $u=u_0$ and $\nabla v(u_0) = \nabla_uf(x_0,u_0)$.
\end{lemma}

\begin{lemma}[The Maximum Theorem {\cite[Theorem 14.2.1]{Syd_05}}]\label{lem:maximum}
Let $\calP$ and $\calX$ be two metric spaces. Consider a maximization problem
\begin{equation}
\max_{x\in \calB(p)} f(p,x), \label{eq:max_problem}
\end{equation}
where $\calB:\calP \rightrightarrows \calX$ is a set-valued map and  $f:\calP\times \calX\to\bbR$ is a function. 
Define the set-valued map $\calS(p) \defeq \argmax_{x\in \calB(p)} f(p,x)$.
If $\calB$ is compact-valued and continuous on $\calP$ and $f$ is continuous on $\calP\times \calX$, and if for some $p_0\in\calP$, $\calS(p_0)=\{s(p_0)\}$, where $s:\calP\to\calX$ is a function, then $s$ is continuous at $p=p_0$. Moreover, we have the same conclusions if the maximization in \eqref{eq:max_problem} is replaced by minimization. 
\end{lemma}

\begin{lemma} \label{lem:Lips_b}
Let $\bz$, $\bz' \in\calZ\subseteq \bbR^m$ and $\bA$, $\bA'\in\calA\subseteq \bbR^{m\times n}$, where both $\calZ$ and $\calA$ are compact sets. Let $\calB$ be a compact set in $\bbR^n$, and define $g:\calB\to\bbR$ as $g(\bb) = 1/2\norm{\bz-\bA\bb}_2^2 - 1/2 \norm{\bz'-\bA'\bb}^2_2$. Then $g$ is Lipschitz on $\calB$ with Lipschitz constant $c_1\norm{\bz-\bz'}_2 + c_2\norm{\bA-\bA'}_2$, where $c_1$ and $c_2$ are two positive constants. In particular, when both $\bz'$ and $\bA'$ are zero, we have that $\tilg(\bb) = 1/2\norm{\bz-\bA\bb}_2^2$ is Lipschitz on $\calB$ with Lipschitz constant $c$ independent of $\bz$ and $\bA$. 
\end{lemma}
\begin{proof}
It suffices to show $\norm{\nabla g(\bb)}_2 \le c_1\norm{\bz-\bz'}_2 + c_2\norm{\bA-\bA'}_2$ for any $\bb\in\calB$ and some positive constants $c_1$ and $c_2$ (independent of $\bb$). We write  $\norm{\nabla g(\bb)}_2$ as
\begin{align*}
\norm{\nabla g(\bb)}_2 &= \normt{{\bA'}^T\left(\bA'\bb - \bz'\right) - \bA^T\left(\bA\bb-\bz\right)}_2\\
& = \normt{({\bA'}^T\bA'-\bA^T\bA)\bb - ({\bA'}^T\bz' - \bA^T\bz)}_2\\
& \le \underbrace{\normt{{\bA'}^T\bA'-\bA^T\bA}_2\norm{\bb}_2}_{\lrang{1}} + \underbrace{\normt{{\bA'}^T\bz' - \bA^T\bz}_2}_{\lrang{2}}.
\end{align*}
By the  compactness of $\calZ$, $\calA$ and $\calB$, there exist positive  constants $M_1$, $M_2$ and $M_3$ such that $\norm{\bz}_2\le M_1$, $\norm{\bA}_2\le M_2$ and $\norm{\bb}_2\le M_3$, for any $\bz\in\calZ$, $\bA\in\calA$ and $\bb\in\calB$. Thus, 
\begin{align*}
\lrang{1} &\le M_3\normt{\bA^T\bA-\bA^T\bA'+\bA^T\bA'-{\bA'}^T\bA'}_2\\
&\le  M_3\lrpar{\normt{\bA^T(\bA-\bA')}_2+\norm{(\bA-\bA')^T\bA'}_2}\\
&\le 2M_2M_3\norm{\bA-\bA'}_2.
\end{align*}
Similarly for $\lrang{2}$ we have
\begin{align*}
\lrang{2} &= \norm{\bA^T\bz- \bA^T\bz'+\bA^T\bz'-{\bA'}^T\bz'}_2\\
&\le \norm{\bA^T(\bz-\bz')}_2+\norm{(\bA-\bA')^T\bz'}_2\\
&\le M_2\norm{\bz-\bz'}_2 + M_1\norm{\bA-\bA'}_2.
\end{align*}
Hence $\lrang{1}+\lrang{2}\le M_2\norm{\bz-\bz'}_2 + (M_1+2M_2M_3)\norm{\bA-\bA'}_2$. We now take $c_1 = M_2$ and $c_2 = M_1+2M_2M_3$ to complete the proof. 
\end{proof}

\begin{lemma}[{\cite{Nitanda_14}}]\label{lem:uni_samp_wo}
Let $\{\bz_i\}_{i=1}^n\subseteq\bbR^d$ and define $\barbz\defeq 1/n\sum_{i=1}^n\bz_i$. Uniformly sample a random subset $\calS$ of $[n]$ with size $b$ without replacement. Then 
\begin{equation}
\bbE_{\calS} \left[\norm{\frac{1}{b}\sum_{i\in\calS}\bz_i - \barbz}^2\right] \le \frac{n-b}{b(n-1)}\left(\frac{1}{n}\sum_{i=1}^n\norm{\bz_i}^2\right). 
\end{equation}
\end{lemma}

\begin{lemma}[Upperbound of a composite function; {\cite{Ghad_16a,Reddi_16}}]\label{lem:Lips_bound}
Consider a composite function $h\defeq h_1+h_2$ defined on a Euclidean space $\calE$, where both $h_1:\calE\to\bbR$ is differentiable with $L$-Lipschitz gradient and $h_2:\calE\to\barbbR_+$ is proper, closed and convex. Fix any $\eta>0$ and $\bx\in\calE$. Let $y = \prox_{\eta h_2}(\bx-\eta\bd)$, for some $\bd\in\calE$. Then 
\begin{align}
h(\by) \le h(\bz)+\lrangle{\by-\bz}{\nabla h_1(\bx)-\bd} + \left(\frac{L}{2}-\frac{1}{2\eta}\right)\norm{\by-\bx}^2 + \left(\frac{L}{2}+\frac{1}{2\eta}\right)\norm{\bz-\bx}^2-\frac{1}{2\eta}\norm{\by-\bz}^2,\;\forall\,\bz\in\calE.\nn
\end{align}
\end{lemma}

\begin{lemma}[{\cite{Hayashi_16}}] \label{lem:inner_prod_bound}
For any two vectors $\ba,\bb$ in a Euclidean space $\calE$ and any $\beta>0$, we have
\begin{equation}
2\lrangle{\ba}{\bb} \le ({1}/{\beta})\norm{\ba}^2 + \beta\norm{\bb}^2. 
\end{equation}
\end{lemma}
\begin{proof}
$(\sqrt{{1}/{\beta}}\ba -\sqrt{\beta}\bb)^2\ge 0$.
\end{proof}

\end{document}